\begin{filecontents*}{example.eps}
\documentclass[smallextended]{svjour3}       % onecolumn (second format)

\usepackage[12pt]{extsizes}

\smartqed  

\usepackage{graphicx}
\usepackage{url}
\usepackage{amsmath,amssymb}
\usepackage{stmaryrd}
\usepackage{enumitem}

\def\free{{free}}
\def\Leibeq{\doteq}

\def\typea{\alpha}

\newcommand{\typearrow}{\shortrightarrow}
\newcommand{\itype}{i}

\newcommand{\proptype}{\tau}
\newcommand{\cjl}{DDL}

\def\Types{{T}}                     %set of type symbols
\def\ar{\rightarrow}

\def\Vars{V}
\def\Consts{C}

\newcommand\entity[1]{\text{\textrm{#1}}}

\newcommand\hol[1]{#1}

\begin{document}

\title{Faithful Semantical Embedding of a Dyadic Deontic Logic in HOL %\thanks{Any thanks?}
}
%\subtitle{An elegant and simple approach}


%\author{}
\author{Christoph Benzm\"uller \and Ali Farjami \and  Xavier Parent}


%\institute{}
 \institute{Christoph Benzm\"uller \at
  University of Luxembourg, Luxembourg, and Freie Universit\"at
  Berlin, Germany\\
   \email{c.benzmueller@gmail.com}  \and
 Ali Farjami, Xavier Parent \at
  University of Luxembourg, Luxembourg \\
   \email{ali.farjami@uni.lu,
     xavier.parent@uni.lu} 
 \\
   %          \emph{Present address:} of F. Author  %  if needed
}
 
\date{\today}

\maketitle

\begin{abstract}
  A shallow semantical embedding of a dyadic deontic logic by Carmo
  and Jones in classical higher-order logic is presented. This
  embedding is proven sound and complete, that is,  faithful.
 
  The work presented here provides the theoretical foundation for the
  implementation and  automation of dyadic deontic logic within
  off-the-shelf higher-order theorem provers and proof assistants.

  % Moreover, to
  % demonstrate its practical relevance, the embedding has 
  % been encoded in the Isabelle/HOL proof assistant.  As a result a sound and complete
  % (interactive and automated) theorem prover for the dyadic deontic
  % logic of Carmo and Jones is obtained. 

  % This prover is 
  % applicable not only for the propositional fragment of this logic but
  % naturally scales for its first-order and higher-order extensions.
   %  The experiments with this environment provide evidence
  % that this logic \textit{implementation} fruitfully enables
  % interactive and automated reasoning at the meta-level and the
  % object-level.

\end{abstract}

%\fbox{\begin{minipage}{30em}
%Issues in notation to be fixed
%\begin{itemize}
%\item symbols use dfor propositional letters not always the same (sometimes upper-case, sometimes lower-case, sometimes with a superscrippt
%\end{itemize}
%\end{minipage}}
\section{Introduction}

%Normative notions such as obligation and permission are the subject of
%deontic logics \cite{gabbay13:_handb_deont_logic_normat_system}, and
%conditional obligations are addressed in so called dyadic deontic
%logic.

Dyadic deontic logic is the logic for reasoning with dyadic obligations (``it ought to be the case that ... if it is the case that ...").
%A logic of contrary-to-duty conditionals
A particular dyadic deontic logic, tailored to so-called contrary-to-duty conditionals, has been proposed by Carmo
and Jones \cite{CJ13}. We shall  refer to it as \cjl\ in
the remainder. \cjl\   comes with a neighborhood semantics and a weakly
complete axiomatization over the class of finite models. The
framework is immune to the well-known contrary-to-duty paradoxes, like Chisholm's paradox, and other related puzzles. 
%proves to be appropriate for handling the well-known paradoxes involving contrary-to-duty obligations like Chisholm's paradox.
%paradoxes raised avoids altogether the problems traditionally associated with
%
%is immune to the paradoxes 
However, the question of how to mechanise and automate reasoning tasks in \cjl\  has not been studied yet. %the reasoning  performed 

This article adresses this challenge. We essentially devise a faithfull
semantical embedding of \cjl\ in classical higher-order logic
(HOL). The latter logic thereby serves as an universal meta-logic.
Analogous to successful, recent work in the area of
computational metaphysics (cf.~\cite{C66} and the references
therein), the key motivation is to mechanise and automate 
\cjl\ on the computer by reusing existing theorem proving technology for meta-logic
HOL. The embedding of \cjl\ in HOL as devised in this article enables
just this.

 % This is demonstrated with the help of a concrete encoding
% of the presented work in the proof assistant Isabelle/HOL
% \cite{Isabelle}.

%This article focuses on the embedding of  propositional \cjl. 
%However, analogous to previous work (cf.~\cite{J23}), the presented
%embedding naturally scales for first-order and even
%higher-order extensions of \cjl. In this regard, the framework as
%presented here may serve as a starting point for 
%much future work. This is particularly relevant since quantified extensions of
%\cjl\ have not been studied yet. 

Meta-logic HOL \cite{AndrewsSEP}, as employed in this article, was
originally devised by Church \cite{Church40}, and further developed by
Henkin \cite{Henkin50}, Andrews
\cite{Andrews:ritt71,Andrews72a,Andrews72b}.  It bases both terms and formulas on simply typed
$\lambda$-terms.  The use of the $\lambda$-calculus
has some major advantages.  For example, $\lambda$-abstractions over
formulas allow the explicit naming of sets and predicates, something
that is achieved in set theory via the comprehension axioms.  Another
advantage is, that the complex rules for quantifier instantiation at
first-order and higher-order types is completely explained via the
rules of $\lambda$-conversion (the so-called rules of $\alpha$-,
$\beta$-, and $\eta$-conversion) which were proposed earlier by Church
\cite{church32am,church36ajm}.  These two advantages are exploited in
our embedding of \cjl\ in HOL.

Different notions of semantics for HOL have been thoroughly studied in
the literature \cite{J6,Muskens07}. In this article we assume HOL with
Henkin semantics (cf. the detailed description by Benzm\"uller
et.~al.~\cite{J6}). For this notion of HOL, which does not suffer from
G\"odel's incompleteness results, several sound and complete theorem
provers have been developed in the past decades \cite{BM14}. We
propose to reuse these systems for the automation of \cjl. The
semantical embedding as devised in this article provides both the
theoretical foundation for the approach and the practical
bridging technology that is enabling \cjl\ applications within existing HOL theorem provers.

% Predicate logic with higher-order quantification was developed first
% by Frege in his Begriffsschrift \cite{Frege1879} and by Russell in his
% ramified theory of types \cite{Russell08}, which was later simplified
% by others, including Chwistek and Ramsey \cite{Ramsey26,Chwistek48},
% Carnap, and finally Church \cite{Church40,AndrewsSEP} in his simple theory of
% types, also referred to as HOL.
%
% HOL bases both terms and formulas on simply typed $\lambda$-terms and
% the equality of terms and formulas is given by equality of such
% $\lambda$-terms.  The use of the $\lambda$-calculus has some major
% advantages.  For example, $\lambda$-abstractions over formulas allow
% the explicit naming of sets and predicates, something that is achieved
% in set theory via the comprehension axioms.  Another advantage is,
% that the complex rules for quantifier instantiation at first-order and
% higher-order types is completely explained via the rules of
% $\lambda$-conversion (the so-called rules of $\alpha$-, $\beta$-, and
% $\eta$-conversion) which were proposed earlier by Church
% \cite{church32am,church36ajm}.  These two advantages
% are exploited in our embedding of \cjl\ in HOL.
%\marginpar{The introduction needs to be rewritten}

The article is structured as follows: Section 2 outlines the syntax
and semantics of \cjl, as far as needed for this article. Section 3
provides a comparably detailed introduction into HOL; this is needed
to keep the article sufficiently self-contained. The semantical
embedding of \cjl\ in HOL is then devised and studied in
Sec.~4. This section also presents the respective soundness and completeness
proofs. % Section 5 then depicts and discusses the concrete
% implementation of the devised embedding in the proof assistent system
% Isabelle/HOL, and 
Section~5 concludes the paper.

\section{The Dyadic Deontic Logic of Carmo and Jones} \label{sec:qcl}

% As an exemplary challenging logic we consider quantified conditional
% logic (QCL). 
% The syntax of QCL is \[\varphi,\psi ::= P \mid
% k(X^1,\ldots,X^n) \mid \neg \varphi \mid \varphi \vee \psi \mid
% \varphi \Rightarrow \psi \mid \forall^{co} X \varphi \mid \forall^{va}
% X \varphi \mid \forall P \varphi\]  $P$ resp.~$X^i$ are 
% propositional resp.~individual variables, and the symbols
% $k$ are $n$-ary constants. $\Rightarrow$ represents the modal
% conditional operator, and $\forall^{co} X \varphi$ and $\forall^{va} X
% \varphi$ denote constant domain and varying domain quantification.
%  Regarding semantics we assume {selection function
%   semantics} \cite{Stalnaker68}. 

% Dyadic deontic logic introduces a conditional operator $O(p/q)$, to be
% read as ``it ought to be the case that p, given q''. Many dyadic
% deontic logics have been proposed to deal with so-called
% contrary-to-duty reasoning, cf.~\cite{Carmo2002} for an overview on
% this area. An example is the dyadic deontic 
% logic proposed by Carmo and Jones called \cjl\ below.

This section provides a concise introduction of \cjl, the dyadic deontic logic
proposed by  Carmo and Jones. Definitions as required for
the remainder are presented. For further details we refer to the literature ~\cite{CJ13,Carmo2002}.

To define the formulas of \cjl\ we start with an countable
set of propositional symbols $P$, and we choose $\neg$ and $\vee$ as the only 
primitive connectives. 

The set of \textit{\cjl\ formulas} is given as the smallest set of
formulas obeying the following conditions:
\begin{itemize}[topsep=1pt,itemsep=0ex,partopsep=1ex,parsep=1ex]
\item Each $p^i\in P$ is an (atomic) \cjl\ formula.
\item Given two arbitrary \cjl\ formulas $\varphi$ and $\psi$, then \\
\begin{tabular}{lcl}
{$\neg \varphi$} & --- & \textit{classical negation}, \\
{$\varphi \vee \psi$} & --- & \textit{classical disjunction}, \\ 
{$\bigcirc(\psi/\varphi)$} & --- & \textit{dyadic deontic obligation: ``it ought to be $\psi$, given
$\varphi$''}, \\
{$\Box \varphi$} & --- & \textit{in all worlds}, \\ 
{$\Box_{a} \varphi$} & --- & \textit{in all actual versions of the
current world}, \\ 
{$\Box_{p} \varphi$} & --- & \textit{in all potential versions of the
current world}, \\
{$\bigcirc_{a}(\varphi)$}  & --- & \textit{monadic deontic operator for
actual obligation}, and \\ 
{$\bigcirc_{p}(\varphi)$} & --- & \textit{monadic deontic
operator for primary obligation} 
\end{tabular}

are also \cjl\ formulas.
\end{itemize}

Further logical connectives can be defined as usual. For example,
we may define $\varphi\wedge\psi := \neg (\neg \varphi \vee \neg \psi)$,
$\varphi \rightarrow \psi := \neg \varphi \vee \psi$,
$\varphi \longleftrightarrow \psi := (\varphi \rightarrow \psi) \wedge
(\psi \rightarrow \varphi)$, $\Diamond \varphi := \neg \Box \neg
\varphi$, $\Diamond_{a} \varphi := \neg \Box_{a} \neg \varphi$,
$\Diamond_{p} \varphi := \neg \Box_{p} \neg \varphi$, 
$\top := \neg q \vee q$, for some propositional symbol $q$,
and $ \bot := \neg \top$.

% \emph{Formulas} of \cjl\ are given by
% the following grammar (where
% $X^i\in\text{IV}, P\in\text{PV}, k\in\text{SYM}$, and where
% $\Rightarrow$ represents conditionality):
% $$
% \varphi,\psi ::= P \mid k(X^1,\ldots,X^n) \mid \neg \varphi \mid \varphi \vee \psi \mid \varphi \Rightarrow \psi \mid \forall^{co} X \varphi \mid \forall^{va} X \varphi \mid  \forall P \varphi  
% $$

%Further connectives, quantifiers, and modal operators can be defined as usual. We also obey the usual definitions of free
%variable occurrences and substitutions.
%\end{definition}

%From the selected set of primitive connectives above, other logical
%connectives can be introduced as abbreviations: for example,
%$\varphi\wedge\psi$, $\varphi \rightarrow \psi$ (material
%implication), $\varphi \longleftrightarrow \psi$, $\neg \varphi \vee
%\psi$ and $(\varphi \rightarrow \psi) \wedge (\psi \rightarrow \varphi)$, respectively.  
 
% For a thorough introduction to selection function semantics we refer to Stalnaker
% \cite{Stalnaker68}.
% We adapt selection
% function semantics for \cjl s.
%which is
%quite similar to Kripke semantics for classical modal logics.

%
%\begin{definition}[Interpretation]
A \cjl\ \emph{model} is a structure $M = \langle S,av ,pv,ob,V 
\rangle$, where $S$ is a  non empty set of items called possible
worlds, $V$ is a function assigning a set of worlds to each atomic
formula, that is,~$ V(p^i)\subseteq S$. $av$: $S \rightarrow P(S) $, where
$P(S)$ denotes the power set of $S$, is a function mapping worlds to
sets of worlds such that $ av(s)\neq \emptyset$. $av(s)$ denotes the
set of actual versions of the world $s$.   
$pv$: $S \rightarrow P(S) $ is another, similar mapping such that
$av(s) \subseteq pv(s)$ and $s \in pv(s)$. $pv(s)$ denotes the set of potential versions of the world $s$.
$ob$: $P (S)  \rightarrow  P(P(S))$, which denotes the set of propositions
that are obligatory in context $\bar{X} \subseteq S$,  is a function mapping set of worlds to
sets of sets of worlds.  The following conditions hold for $ob$ (where
$\bar{X},\bar{Y},\bar{Z}$ designate arbitrary subsets of $S$):  

 \begin{enumerate}[topsep=1pt,itemsep=0ex,partopsep=1ex,parsep=1ex]
 	\item
 	$\emptyset \notin ob(\bar{X}). $
 	\item
 	$\text{If } \bar{Y} \cap \bar{X} = \bar{Z} \cap \bar{X}
        \text{, then } \bar{Y} \in ob(\bar{X}) \text{ if and only if } \bar{Z} \in ob(\bar{X}).$
 	\item
 	Let $\bar{\beta} \subseteq ob(\bar{X}) $ and $ \bar{\beta}
        \neq \emptyset $. If $ (\cap \bar{\beta} ) \cap \bar{X}\neq
        \emptyset$ (where $ \cap \bar{\beta} = \{s \in S\ | \mbox{ for all } \bar{Z}\in \beta \mbox{ we have } s\in \bar{Z} \} $), then $ ( \cap \bar{\beta}) \in ob(\bar{X})$.
 	\item
 	$\text{If }  \bar{Y}\subseteq \bar{X} \text{ and }  \bar{Y}
        \in ob(\bar{X}) \text{ and } \bar{X} \subseteq \bar{Z}, \text{
          then } (\bar{Z} \smallsetminus \bar{X})\cup \bar{Y} \in ob(\bar{Z}).$ 
 	\item
 	$\text{If }\bar{Y} \subseteq \bar{X} \text{ and }  \bar{Z} \in
        ob(\bar{X}) \text{ and } \bar{Y} \cap \bar{Z} \neq \emptyset,
        \text{ then } \bar{Z} \in ob(\bar{Y}).	 $ 
 \end{enumerate}	

%\end{definition}
%
%\begin{definition}[Variable Assignment]

%\end{definition}
%

%\begin{definition}[Satisfiability] \label{def:satisfiability}
\emph{Satisfiability} of a formula $\varphi$ for a model $M = \langle
S,av ,pv,ob,V \rangle$ and a world $s \in S$ is denoted by $M,s
\models \varphi$ and we define $ V^{M}(\varphi) = \{ s \in S\ |\
M,s\models \varphi \}$. In order to simplify the presentation,
whenever the model $ M $ is obvious from context, we write $
V(\varphi)  $ instead of $ V^{M}(\varphi)$. Moreover, we often use ``iff''
as shorthand for ``if and only if''.
\[
\begin{tabular}{lcl}
$M,s \models\; p$  & iff  &  $s \in  V(p)$
\\
$M,s \models\; \neg \varphi$ &  iff   &  $M,s \not\models \varphi \, (\text{that is, not } M,s \models \varphi)$
\\ 
$M,s \models\; \varphi \vee \psi$ &  iff   &  $M,s \models \varphi \text{ or } M,s \models \psi$
\\
$M,s \models\; \Box \varphi$  &   iff   &  $V(\varphi) = S$ 
\\
$M,s \models\; \Box_{a} \varphi$ &   iff   &  $av (s) \subseteq V(\varphi)$ 
\\
$M,s \models\; \Box_{p} \varphi$ &  iff   &  $pv (s) \subseteq V(\varphi)$ 
\\
$M,s \models\; \bigcirc (\psi /\varphi)$ &   iff   &  $V(\psi) \in ob(V(\varphi))$ 
\\
$M,s \models\; \bigcirc_{a} \varphi$  &   iff   &  $V(\varphi ) \in ob(av(s)) \text{ and } av(s)\cap V(\neg \varphi)\neq\emptyset$ 
\\
$M,s \models\; \bigcirc_{p} \varphi$ &  iff  &  $V(\varphi ) \in ob(pv(s))  \text{ and }  pv(s)\cap V(\neg \varphi)\neq\emptyset$
\end{tabular}
\]
Our evaluation rule for $\bigcirc (\_/\_)$ is a simplified version of
the one used by Carmo and Jones. Given the constraints placed on $ob$,
both rules are equivalent (cf. \cite[result II-2-2]{J31}).

As usual, a \cjl\ formula $\varphi$ is \emph{valid in a \cjl\
  model} $M = \langle S,av ,pv,ob,V \rangle$, denoted as
$M \models^{\cjl} \varphi$, if and only if for all worlds $s \in S$  holds $M,s \models \varphi$. A formula
$\varphi$ is \emph{valid}, denoted $\models^{\cjl} \varphi$, if and
only if it is valid in every \cjl\ model.
%\end{definition}
% QCLs have many applications, including action planning, counterfactual
% reasoning, default reasoning, deontic reasoning, metaphysical modeling
% and reasoning about knowledge. While there is broad literature on
% propositional conditional logics only a few authors have addressed
% first-order extensions
% \cite{delgrande98,DBLP:journals/tocl/FriedmanHK00}.  

\section{Classical Higher-order Logic} \label{sec:hol} In this section
we introduce classical higher-order logic (HOL). The presentation,
which has partly 
been adapted from \cite{J31}, is rather detailed in order to keep the
article sufficiently self-contained.

%in \S \ref{sec:embedding}
% Higher-order substitution can be seen (from the inference step
% point-of-view) as one step, but it can pack a significant
% computational influence on a formula: normalization of $\lambda$-terms
% can be explosive (and expressive), in particular, since
% $\lambda$-terms may contain logical connectives and quantifiers.
% Bindings are also treated uniformly for all structures and terms that
% have bindings.  For example, if $p$ is a variable of predicate type
% and $A$ is the formula $\alldot{p} B(p)$, then the universal
% instantiation of $A$ with the term, say, $t$, namely the formula
% $[t/p]B$, can be a formula with many more occurrences of logical
% connectives and quantifiers than there are in the formula $B(p)$.

\subsection{Syntax of HOL}
For defining the syntax of HOL, we first introduce the set ${T}$
of \emph{simple types}. We assume that ${T}$ is freely generated from a set
of \emph{basic types} $BT \supseteq \{o,i\}$ using the function type
constructor $\typearrow$. $o$ denotes the (bivalent) set of Booleans,
and $i$ a non-empty set of individuals. 
%\begin{definition}

For the definition of HOL, we start out with a family of denumerable
sets of typed constant symbols $(\Consts_\alpha)_{\alpha\in T}$,
called \emph{signature},  and a
family of denumerable sets of typed variable symbols
$(\Vars_\alpha)_{\alpha\in T}$.\footnote{ For example in Section 4 we
  will  assume constant symbols $av$, $pv$ and $ob$ with types $i
  \typearrow i \typearrow o$, $i \typearrow i \typearrow o$ and $(i \typearrow o) \typearrow (i \typearrow o) \typearrow o$ as part of the signature.} We employ Church-style typing, where
each term $t_\alpha$ explicitly encodes its type information in
subscript $\alpha$. 

The \emph{language of HOL} is given as the smallest
set of terms obeying the following conditions. 

\begin{itemize}[topsep=1pt,itemsep=0ex,partopsep=1ex,parsep=1ex]
\item Every typed constant
symbol $c_{\alpha}\in \Consts_\alpha$ is a  HOL term of type $\alpha$.
\item Every typed variable symbol
$X_\alpha\in \Vars_\alpha$ is a  HOL term of type $\alpha$. 
\item If
$s _{\alpha \typearrow \beta}$ and $t _{\alpha}$ are HOL terms of types
${\alpha \typearrow \beta}$ and ${\alpha}$, respectively, then
$(s_{\alpha \typearrow \beta}\, t_\alpha)_{\beta}$, called
\emph{application}, is an HOL term of type $\beta$. 
\item If
$X_\alpha\in \Vars_\alpha$ is a typed variable symbol and $s_{\beta}$
is an HOL term of type $\beta$, then
$(\lambda X_{\alpha}s_{\beta})_{\alpha \typearrow \beta}$, called
\emph{abstraction}, is an HOL term of type
${\alpha \typearrow \beta}$. 
\end{itemize}

The above definition encompasses the simply typed
$\lambda$-calculus. In order to extend this base framework into logic HOL we
simply ensure that the signature $(\Consts_\alpha)_{\alpha\in T}$ provides
a sufficient selection of primitive logical connectives. Without
loss of generality, we here assume the following \emph{primitive
  logical connectives} to be part of the signature:
$\neg_{o \typearrow o} \in\Consts_{o \typearrow o}$,
$\vee_{o \typearrow o \typearrow o} \in\Consts_{o \typearrow o \typearrow o}$, 
$\Pi_{(\typea \typearrow o) \typearrow o} \in\Consts_{(\typea \typearrow o) \typearrow o}$
and $=_{\alpha\typearrow \alpha \typearrow \alpha}\in\Consts
_{\alpha\typearrow \alpha \typearrow \alpha}$, abbreviated as
$=^\alpha$. The symbols $\Pi_{(\typea \typearrow o)
  \typearrow o}$ and $=_{\alpha\typearrow \alpha \typearrow \alpha}$
are generally assumed for each type $\typea \in \Types$.
The denotation of the primitive
  logical connectives is fixed below according to their intended meaning. 
\emph{Binder notation} $\forall X_{\alpha} \,s_o$ is used as an abbreviation
for $\Pi_{(\alpha \typearrow o)\typearrow o}\lambda X_{\alpha}
s_{o}$. Universal quantification in HOL is thus modeled
with the help of the logical
constants $\Pi_{(\alpha \typearrow o)\typearrow o}$ to be used in
combination with lambda-abstraction. That is, the only binding mechanism provided in HOL is
lambda-abstraction. 

HOL is a logic of terms in the sense that the \emph{formulas of HOL} are
given as the terms of type $o$. In addition to the primitive logical connectives selected above, we could
assume \emph{choice operators}
$\epsilon_{(\alpha\typearrow o)\typearrow
  \alpha}\in\Consts_{(\alpha\typearrow o)\typearrow \alpha}$
(for each type $\alpha$) in the
signature. We are not pursuing this here.

Type information as well as brackets may be
omitted %in the remainder
if obvious from the context, and we may also use infix notion to
improve readability.  For example, we may write $(s \vee t)$
instead of $((\vee_{o \typearrow o \typearrow o} s_o) t_o)$.

From the selected set of primitive connectives, other logical
connectives can be introduced as abbreviations.\footnote{As demonstrated by Andrews
  \cite{AndrewsSEP}, we could in fact start out with only
  primitive equality in the signature (for all types $\typea$) and
  introduce all other logical connectives as abbreviations based on
  it. The motivation for the redundant signature as selected here is
  to stay close to the the choices taken in implemented theorem
  provers such as LEO-II and Leo-III and also to theory paper
  \cite{J6}, which is recommended for further details.} For example,
we may define
$s\wedge t := \neg (\neg s \vee \neg t)$,
$s \rightarrow t := \neg s \vee t$,
$s \longleftrightarrow t := (s \rightarrow t) \wedge
(t \rightarrow s)$
, $ \top := (\lambda {X_i} X_i) = (\lambda {X_i} X_i)$, $\bot := \neg \top$ and
$\exists X_\typea s := \neg \forall X_\typea \neg s$.

Also equality can be defined in HOL by exploiting Leibniz' principle, expressing that two 
objects are equal if they share the same properties.
\emph{Leibniz equality} $\Leibeq^\alpha$ at type $\alpha$ is thus defined as $ s_\typea \Leibeq^\alpha t_\typea :=
\forall P_{\typea\typearrow o} (\neg P s \vee  P t)$. 

%\marginpar{Check which of the notions below are actually needed}
Each occurrence of a variable in a term is either bound by a $\lambda$
or free.  We use $\free(s)$ to denote the set of variables with a free occurrence in $s$.
We consider two terms to be \emph{equal} if the terms are the same up
to the names of bound variables, that is, we consider $\alpha$-conversion
implicitly.  %A term $s$ is \emph{closed} if $\free(s)$ is empty.

\emph{Substitution} of a term $s_{\alpha}$ for a variable $X_{\alpha}$
in a term $t_{\beta}$ is denoted by $[s/X]t$. Since we consider
$\alpha$-conversion implicitly, we assume the bound variables of $t$ to
avoid variable capture.

Well-known operations and relations on HOL terms include
\emph{$\beta\eta$-normaliza\-tion} and \emph{$\beta\eta$-equality}, denoted by $s
=_{\beta\eta} t$, \emph{$\beta$-reduction} and \emph{$\eta$-reduction}. A
\emph{$\beta$-redex} $(\lambda X s)t$ $\beta$-reduces to $[t/X]s$. An
\emph{$\eta$-redex} $\lambda X (s X)$, where $X\not\in\free(s)$, $\eta$-reduces to $s$.  We write $s =_{\beta} t$ to mean $s$ can
be converted to $t$ by a series of $\beta$-reductions and
expansions. Similarly, $s=_{\beta\eta} t$ means $s$ can be converted
to $t$ using both $\beta$ and $\eta$.

\subsection{Semantics of HOL}
The semantics of HOL is well understood and thoroughly documented. The
introduction provided next focuses on the aspects as
needed for this article. For more details we refer to the previously mentioned
literature \cite{J6}.

The semantics of choice for the remainder is Henkin
semantics, i.e., we work with Henkin's general models.  Henkin models
(and standard models) are introduced next. We start out with
introducing frame
structures.
% Henkin semantics for HOL closely follows Andrews
% \cite{Andrews72b,AndrewsSEP}.

%\begin{definition}\label{homlframe}
A \emph{frame} $D$ is a collection $\{D_\alpha\}_{\alpha\in\entity{T}}$
of nonempty sets $D_\alpha$, such that $D_o = \{T,F\}$
(for truth and falsehood).  The
$D_{\alpha\ar\beta}$ are collections of functions mapping
$D_\alpha$ into $D_\beta$. 
%\end{definition}

%\begin{definition}\label{holmodel}
A \emph{model} for HOL is a tuple
$\hol{M}=\hol{\langle D, I \rangle}$, where $\hol{D}$ is a frame, and
$\hol{I}$ is a family of typed interpretation functions mapping
constant symbols $\hol{p_\alpha}\in\Consts_\alpha$ to appropriate
elements of $\hol{D_\alpha}$, called the \emph{denotation of
  $\hol{p_\alpha}$}. The logical connectives $\hol{\neg}$,
$\hol{\vee}$, $\hol{\Pi}$ and $\hol{=}$ are always given their
expected, standard denotations:\footnote{Since $=_{\typea\ar\typea\ar o}$ (for all types $\typea$) is in the signature, it is ensured
  that the domains $\hol{D_{\alpha\ar\alpha\ar o}}$ contain the
  respective identity relations. This addresses an issue discovered by Andrews
  \cite{Andrews72b}: if such identity relations were not existing in the
  $\hol{D_{\alpha\ar\alpha\ar o}}$, then Leibniz equality in Henkin
  semantics may not denote as intended.
}

\begin{itemize}[topsep=1pt,itemsep=0ex,partopsep=1ex,parsep=1ex] \label{standard-denotation}
\item $I(\neg_{o \ar o}) = not \in D_{o \ar o}$ such that $not(T) = F$ and $not(F) =
T$.
\item $I(\vee_{o \ar o \ar o}) = or \in D_{o \ar o \ar o}$  such that $or(a,b) = T$ iff ($a = T$ or $b = T$).
\item $I(=_{\typea \ar \typea  \ar o}) = id \in D_{\typea \ar \typea
    \ar o}$  such that for all $a,b\in D_\typea$, $id(a,b) = T$ iff $a$ is identical to $b$.
\item $I(\Pi_{(\typea \ar o)  \ar o}) = all \in D_{(\typea \ar o)  \ar
    o}$  such that for all $s\in D_{\typea \ar o}$, $all(s) = T$ iff $s(a) = T$ for all $a\in D_\typea$; i.e., $s$ is the set
of all objects of type $\typea$.
\end{itemize}

%\end{definition}

\noindent
Variable assignments are a technical aid for the subsequent definition
of an interpretation function $\| . \|^{M,g}$ for HOL terms. This
interpretation function is parametric over a model $M$ and a variable
assignment $g$.

%\begin{definition}\label{homlassignment}
A \emph{variable assignment} $g$ maps
variables $X_\alpha$ to elements in $D_\alpha$. $g[d/W]$ denotes the
assignment that is identical to $g$, except for variable $W$, which is
now mapped to $d$.
%\end{definition}

%\begin{definition}\label{holvalue} \sloppy
  The \emph{denotation} $\hol{\| s_\alpha\|^{M,g}}$ of an HOL term
  $\hol{s_\alpha}$ on a model $\hol{M}=\hol{\langle D, I \rangle}$ under assignment $\hol{g}$ is an element $\hol{d}\in \hol{D_\alpha}$
  defined in the following way: 
\begin{center}
\begin{tabular}{rcl}
$\hol{\|p_\alpha\|^{M,g}}$ & $=$ & $\hol{I(p_\alpha)}$ \\[.5em]
$\hol{\|X_\alpha\|^{M,g}}$ & $=$ & $\hol{g(X_\alpha)}$ \\[.5em]
$\hol{\|(s_{\alpha\ar\beta}\, t_\alpha)_\beta\|^{M,g}}$ & $=$ &
                                                                $\hol{\|s_{\alpha\ar\beta}\|^{M,g}(\|t_\alpha\|^{M,g})}$ \\[.5em]
$\hol{\|(\lambda{X_\alpha} s_\beta)_{\alpha\ar\beta}\|^{M,g}}$ & $= $ &
  the function $\hol{f}$ from $\hol{D_\alpha}$ to $\hol{D_\beta}$ such
  that \\
 & & $\hol{f(d)} = \hol{\|s_\beta\|^{M,g[d/X_\alpha]}}$ for all
  $\hol{d}\in \hol{D_\alpha}$
\end{tabular}
\end{center}
%\end{definition}
%\footnotetext{Since we assumed $\neg_{o\ar o} $ and $\vee_{o\ar o\ar o}$ as primitive logical connectives, it is easy to check that we can derive item 7, 8 and 9 from other denotations.}

%\begin{definition}\label{homlhenkinmodel}
A model $M=\langle D, I \rangle$ is called a \emph{standard model} if
and only if
for all $\alpha,\beta\in T$ we have
$D_{\alpha\ar\beta} = \{ f \mid f : D_\alpha \longrightarrow D_\beta
\}$.
In a \emph{Henkin model (general model)} function spaces are not
necessarily full. Instead it is only required that for all
$\alpha,\beta\in T$,
$D_{\alpha\ar\beta} \subseteq \{ f \mid f : D_\alpha \longrightarrow
D_\beta \}$.
However, it is required that the valuation function $\|\cdot\|^{M,g}$
from above is total, so that every term denotes. Note that this
requirement, which is called \emph{Denotatpflicht}, ensures that the
function domains $D_{\alpha\ar\beta}$ never become too sparse, that is,
the denotations of the lambda-abstractions as devised above are
always contained in them.  % We consider Henkin models in the remainder.
%\end{definition}

\begin{corollary} For any Henkin model $M=\langle D, I \rangle$ and 
  variable assignment $g$ holds:

\begin{enumerate}[topsep=1pt,itemsep=0ex,partopsep=1ex,parsep=1ex]
\item $\hol{\|(\neg_{o\ar o}\, s_o)_o\|^{M,g}}  = \hol{T}$ \quad  iff \quad
   $\hol{\|s_o\|^{M,g}} = \hol{F}$. 
\item $\hol{\|(({\vee_{o\ar o\ar o}}\,s_o)\,t_o)_o\|^{M,g}} =
  \hol{T}$\quad  iff \quad $\hol{\|s_o\|^{M,g}} = \hol{T}$ or $\hol{\|t_o\|^{M,g}}
  = \hol{T}$. 
\item $\hol{\|(({\wedge_{o\ar o\ar o}}\,s_o)\,t_o)_o\|^{M,g}} =
 \hol{T}$\quad  iff \quad $\hol{\|s_o\|^{M,g}} = \hol{T}$ and $\hol{\|t_o\|^{M,g}}
 = \hol{T}$. 

\item $\hol{\|(({\rightarrow_{o\ar o\ar o}}\,s_o)\,t_o)_o\|^{M,g}} =
\hol{T}$ \quad  iff \quad  $\hol{\|s_o\|^{M,g}} = \hol{T}$ then $\hol{\|t_o\|^{M,g}}
= \hol{T}$. 

\item $\hol{\|(({\longleftrightarrow_{o\ar o\ar o}}\,s_o)\,t_o)_o\|^{M,g}} =
\hol{T}$\quad  iff \quad $\hol{\|s_o\|^{M,g}} = \hol{T}$ iff $\hol{\|t_o\|^{M,g}}
= \hol{T}$.  

\item $\hol{\|\top\|^{M,g}}  = \hol{T}$. 

\item $\hol{\|\bot\|^{M,g}}  = \hol{F}$. 

\item $\hol{\|(\forall X_\alpha s_o)_o\|^{M,g}} = \hol{\|(\forall_{(\alpha\ar o)\ar o}(\lambda{X_\alpha}
   s_o))_o\|^{M,g}} = \hol{T}$\quad  iff \quad  for all $\hol{d}\in
 \hol{D_\alpha}$ we have $\hol{\|s_o\|^{M,g[d/X_\alpha]}} = \hol{T}$.

\item $\hol{\|(\exists X_\alpha s_o)_o\|^{M,g}} = \hol{T}$\; iff \; there
  exists $\hol{d}\in
 \hol{D_\alpha}$ such that $\hol{\|s_o\|^{M,g[d/X_\alpha]}} = \hol{T}$.
\end{enumerate}
\begin{proof}
 We leave the proof as an  exercise to the reader.
\end{proof}
\end{corollary}

 % Truth in a model, validity in a model $M$ and validity are
 % defined as usual.
 An HOL formula $s_o$ is \emph{true} in an Henkin model $M$ under
 assignment $g$ if and only if $\|s_o\|^{M,g} = T$;
 this is also denoted by $M,g \models^\text{HOL} s_o$.  An HOL formula
 $s_o$ is called \emph{valid} in $M$, which is denoted by
 $M\models^\text{HOL} s_o$, if and only if
 $M,g \models^\text{HOL} s_o$ for all assignments $g$. Moreover, a
 formula $s_o$ is called \emph{valid}, denoted by
 $\models^\text{HOL} s_o$, if and only if $s_o$ is valid in all
 Henkin models $M$. Finally, we define $\Sigma\models^\text{HOL} s_o$ for a set of HOL
 formulas $\Sigma$ if and only if $M\models^\text{HOL} s_o$ for
 all Henkin models $M$
 with $M\models^\text{HOL} t_o$ for all $t_o\in \Sigma$.

Note that any standard model is
obviously also a Henkin
model. Hence, validity of a HOL formula $s_o$ for all Henkin models,
implies  validity of $s_o$ for all standard models.

\section{Modeling \cjl\ as a Fragment of HOL} 
% Regarding the particular choice of HOL, we here assume a set of basic
% types $BT=\{o,i\}$, where $o$ denotes the type of Booleans as before.
% Without loss of generality, $i$ is now identified with a (non-empty)
% set of worlds.

This section, as the core contribution of this article, presents a
shallow semantical embedding of \cjl\ in HOL and proves its soundness and
completeness.

\subsection{Semantical Embedding}
\cjl\ formulas are identified in our semantical embedding with certain
HOL terms (predicates) of type $i \typearrow o$. They can be applied
to terms of type $i$, which are assumed to denote possible worlds.
That is, the HOL type $i$ is now identified with a (non-empty) set of
worlds. 
Type $i \typearrow o$ is abbreviated as $\proptype$ in the
remainder. The HOL signature is assumed to contain the constant symbol
$ av_{\itype \typearrow \proptype }$,
$ pv_{\itype \typearrow \proptype}$ and
$ ob_{\proptype \typearrow \proptype \typearrow o}$. Moreover, for each propositional symbol $p^i$ of \cjl, the HOL
signature must contain a respective constant symbols
$p^i_\tau$. Without loss of generality, we assume that besides
those symbols and the primitive logical connectives of HOL, no 
other constant symbols are given in the signature of HOL. 

The mapping $\lfloor \cdot \rfloor$ translates \cjl\ formulas
$s$ into HOL terms $\lfloor s \rfloor$ of type
$\proptype$. The mapping is recursively defined:
\[
\begin{array}{lcl}
\lfloor p^i \rfloor &=& p^i_\proptype \\
%\lfloor k(X^1,\ldots,X^n) \rfloor % &=& \lfloor k \rfloor \lfloor X^1 \rfloor \ldots \lfloor X^n \rfloor \\
%                             &=& k_{u^n \typearrow \proptype}\,X^1_\utype \ldots X^n_\utype\\
\lfloor \neg s \rfloor &=&  \neg_\proptype\,\lfloor s \rfloor \\
\lfloor s\vee t \rfloor &=&  \vee_{\proptype \typearrow \proptype \typearrow \proptype}\,\lfloor s \rfloor \lfloor t \rfloor\\
\lfloor   \Box s   \rfloor &=&  \Box_{ \proptype \typearrow \proptype }\,\lfloor s \rfloor \\
\lfloor  \bigcirc(t/s)  \rfloor &=&  \bigcirc_{\proptype \typearrow \proptype \typearrow \proptype}\,\lfloor s \rfloor \lfloor t \rfloor \\
\lfloor   \Box_{a} s   \rfloor &=&  \Box^{a}_{\proptype\typearrow \proptype  }\,\lfloor s \rfloor \\
\lfloor   \Box_{p} s   \rfloor &=&  \Box^{p}_{\proptype\typearrow \proptype  }\,\lfloor s \rfloor \\
\lfloor  \bigcirc_{a}(s)  \rfloor &=&  \bigcirc^{a}_{\proptype \typearrow \proptype  }\,\lfloor s \rfloor  \\
\lfloor  \bigcirc_{p}(s)  \rfloor &=&  \bigcirc^{p}_{\proptype \typearrow \proptype  }\,\lfloor s \rfloor  \\
% \end{array}\]
% \[
% \begin{array}{lcl}
\end{array}
\]
 $\neg_{\proptype \typearrow \proptype} $, $\vee_{\proptype \typearrow
  \proptype \typearrow \proptype}$, $\Box_{\proptype \typearrow
  \proptype }\ $, $\bigcirc_{\proptype \typearrow \proptype \typearrow
  \proptype}\ $,  $ \Box^{a}_{\proptype \typearrow \proptype  }\ $, $
\Box^{p}_{ \proptype \typearrow \proptype}\ $, $
\bigcirc^{a}_{\proptype \typearrow \proptype  }\ $ and
$\bigcirc^{p}_{\proptype \typearrow \proptype}\ $ 
% realize the \cjl\
% connectives in
% HOL. They 
thereby abbreviate the following HOL terms:
%\footnote{Note the predicate argument $A$ of $f$ in the term for $\Rightarrow_{\proptype \typearrow \proptype \typearrow \proptype}$ and the second-order quantifier $\forall P_{\proptype}$ in the term for $\Pi_{(\proptype\typearrow\proptype)\typearrow\proptype}$. FOL encodings of both constructs, if feasible, will be less natural.}
\[
\begin{array}{ll}
\neg_{\proptype \typearrow \proptype} &= \lambda A_\proptype \lambda X_\itype \neg(A\,X)\\
\vee_{\proptype \typearrow \proptype \typearrow \proptype} &= \lambda A_\proptype \lambda B_\proptype \lambda X_\itype (A\,X \vee B\,X)\\
\Box_{\proptype \typearrow \proptype } & = \lambda A_\proptype
                                         \lambda X_\itype \forall
                                         Y_{\itype} (A\,Y) \\
\bigcirc_{\proptype \typearrow \proptype \typearrow \proptype} &=
                                                                 \lambda
                                                                 A_\proptype
                                                                 \lambda
                                                                 B_\proptype
                                                                 \lambda
                                                                 X_\itype
                                                                 (ob\,A\,B)\\
\Box^{a}_{\proptype\typearrow \proptype }\ & = \lambda A_\proptype
                                             \lambda X_\itype
                                             \forall{Y_\itype} (\neg
                                             (av\,X\,Y) \vee A\,Y)   \\
\Box^{p}_{\proptype \typearrow \proptype }\ & = \lambda A_\proptype  \lambda X_\itype      \forall{Y_\itype} (\neg (pv\,X\,Y) \vee (A\,Y)) \\
\bigcirc^{a}_{\proptype \typearrow \proptype  }\ & = \lambda
                                                   A_\proptype \lambda
                                                   X_\itype
                                                   ((ob\,(av\,X)\,A)
                                                   \wedge \exists
                                                   Y_\itype (av\, X\,
                                                   Y \wedge \neg (A\, Y) )) \\
\bigcirc^{p}_{\proptype \typearrow \proptype  }\ & = \lambda
                                                   A_\proptype \lambda
                                                   X_\itype
                                                   ((ob\,(pv\,X)\,A)
                                                   \wedge \exists
                                                   Y_\itype (pv\, X\,
                                                   Y \wedge \neg (A\, Y) )) \\
\end{array}
\]

Analyzing the truth of a translated formula
$\lfloor s \rfloor$ in a world represented by term $w_\itype$
corresponds to evaluating the application
$(\lfloor s \rfloor\,w_\itype)$.  In line with previous work
\cite{J23}, we define
$ \text{vld}_{\proptype\typearrow o} = \lambda A_\proptype \forall
S_\itype (A\, S)$.
With this definition, validity of a \cjl\ formula $s$ in \cjl\
corresponds to the validity of formula
$(\text{vld}\, \lfloor s \rfloor)$ in HOL, and vice versa.

\subsection{Soundness and Completeness}
To prove the soundness and completeness, that is, faithfulness, of the above embedding, a
mapping from \cjl\ models into Henkin models is employed.

% Soundness and completeness of this embedding of QCL in HOL has been established in \cite{C37}.

% \begin{theorem}[Soundness and Completeness of the Embedding]  $\models^{QCL} s \quad \text{iff} \quad  \models^{\text{HOL}} \text{vld}\, s_\proptype$
% \end{theorem}

%\begin{definition}[Mapping of Assignments]
%\end{definition}

\begin{definition}[Henkin model ${H}^{M}$ for \cjl\ model $M$]\label{def:hm}
\label{embedding} \sloppy
For any \cjl\ model $M = \langle S,av ,pv,ob,V \rangle$, we
define corresponding Henkin models $H^M$. Thus, let a \cjl\ model $M =
\langle S,av ,pv,ob,V \rangle$ be given. Moreover, assume that 
$p^{1},...,p^{m} \in P$, for $ m \geq 1 $, are the only propositional
symbols of  \cjl.
Remember that our embedding requires the corresponding signature of
HOL to provide constant symbols $p_{\proptype}^{j}$
such that $\lfloor p^{j} \rfloor = p_{\proptype}^{j}$ for $ j=1,\ldots,m$. 

An Henkin
model ${H}^{M} = \langle \{D_\alpha\}_{\alpha \in {T}}, I \rangle$ for
$M$ is now defined as follows: $D_\itype$ is chosen as the set of possible
worlds $S$; all other sets $D_{\alpha\typearrow\beta}$ are
chosen as (not necessarily full) sets of functions from $D_\alpha$ to
$D_\beta$.  For all $D_{\alpha\typearrow\beta}$ the rule that
every term $ t_{\alpha\typearrow\beta} $ must have a denotation in
$D_{\alpha\typearrow\beta}$ must be obeyed (Denotatpflicht). In particular, it is required that $D_{\proptype}$, $D_{i\typearrow\proptype}$ and $D_{ \proptype \typearrow \proptype \typearrow o}$ contain the elements $I p_{\proptype}^{j}$, $Iav_{i\typearrow\proptype}$, $I
pv_{i\typearrow\proptype}$ and $Iob_{\proptype \typearrow \proptype
  \typearrow o}$. The interpretation function $I$ of ${H}^{M}$ is
defined as follows:

\begin{enumerate}[topsep=1pt,itemsep=0ex,partopsep=1ex,parsep=1ex]
\item
For $i= 1,\ldots,m $, $ I p_{\proptype}^{i} \in
D_{\proptype} $ is chosen such that $I p_{\proptype}^{i}(s) = T $ iff $s\in
V(p^{j})$ in $M$. % and $I p_{\proptype}^{j}(s) = F $ otherwise.
\item
 $Iav_{i\typearrow\proptype} \in D_{i\typearrow\proptype} $ is chosen 
 such that $I av_{i\typearrow\proptype}(s,u) = T $ iff $u \in av(s)$
 in $M$. % and  $Iav_{i\typearrow\proptype}(s,u) = F $ otherwise.
\item
 $Ipv_{i\typearrow\proptype} \in D_{i\typearrow\proptype} $ is chosen
 such that $I pv_{i\typearrow\proptype}(s,u) = T $ iff $ u\in pv(s) $
 in $M$. % and  $Ipv_{i\typearrow\proptype}(s,u) = F $ otherwise.
\item  
$Iob_{\proptype \typearrow \proptype \typearrow o} \in D_{ \proptype
  \typearrow \proptype \typearrow o} $ is chosen such that
$Iob_{\proptype \typearrow \proptype \typearrow o} (\bar{X},\bar{Y}) =
T  $ iff $ \bar{Y} \in ob(\bar{X}) $ in $M$. % and $Iob_{\proptype \typearrow \proptype \typearrow o} (\bar{X},\bar{Y}) = F $ otherwise.
\item For the logical connectives $\neg$, $\vee$, $\Pi$ and $=$ of
  HOL the interpretation function $I$ is defined as usual (see the
  previous section).
%\item  For all other constants $c_\alpha$, choose $I c_\alpha$ arbitrary.\footnote{In fact, it may be safely assumed that there are
 % no other typed constant symbols given, except for the symbols $f_{i
%    \typearrow \proptype \typearrow \proptype}$, $\exInW_{i\typearrow\utype\typearrow o}$,
%  $k_{u^n\typearrow\proptype}$, and the logical connectives.}
\end{enumerate}

Since we assume that there are no other
symbols (besides the $p^i_\tau$, $av$, $pv$, $ob$
and $\neg$, $\vee$, $\Pi$, and $=$)  in the signature of HOL, 
$I$ is a total function.
Moreover, the above construction guarantees
that ${H}^{M}$ is a Henkin model: $\langle D,I \rangle$ is a frame,
and the choice of $I$  in combination with the Denotatpflicht ensures
that for arbitrary assignments $g$,  $\|.\|^{{H}^{M},g}$ is an
total evaluation function.

%\footnote{In
%  ${H}^{M}$ we have merely fixed $D_\itype=S$ and the interpretations of
%  $p_{\proptype}^{j}$, $av_{i\typearrow\proptype}$, $ pv_{i\typearrow\proptype} $ and $ob_{\proptype \typearrow \proptype \typearrow o}$.  These
%  choices are not in conflict with any of the requirements regarding
%  frames and interpretations. The existence of a valuation function
%  ${V}$ for an HOL interpretation crucially depends on how sparse the
%  function spaces have been chosen in frame $\{D_\alpha\}_{\alpha \in
%    {T}}$. Andrews \cite{Andrews72b} discusses criteria that are sufficient
%  to ensure the existence of a valuation function; they require that
%  certain $\lambda$-abstractions have denotations in frame
%  $\{D_\alpha\}_{\alpha \in {T}}$.\label{foot}}
\end{definition}

% \begin{proof}
%    In ${H}^{M}$ we have merely fixed $D_i$ and the
%    interpretation of the constant symbols $p^j_\proptype$ and
%    $f_{i \typearrow \proptype \typearrow \proptype}$.
%    These constraints are obviously not in conflict with any of our
%    requirements in Definitions \ref{stt_frame} and
%    \ref{stt_interpretation}. ${H}^{M}$ is trivially
%    a standard model since the function spaces have been chosen
%    full. As mentioned before, it is easy to verify that every standard
%    model is also a Henkin model (general model) and that the valuation
%    function ${V}$ is uniquely defined
%    (cf. also~\cite{Andrews1972b,AndrewsSEP}). (Remark: The existence of
%    a valuation function ${V}$ for an HOL interpretation
%    crucially depends on how sparse the function spaces have been
%    chosen in frame $D_\alpha\}_{\alpha \in
%      {T}}$. \cite{Andrews1972b} discusses criteria that are
%    sufficient to ensure the existence of a valuation function; they
%    require that certain $\lambda$-abstractions have denotations in
%    frame $\{D_\alpha\}_{\alpha \in {T}}$. Since our function
%    spaces are full, this is trivially the case for
%    ${H}^{M}$.)

% %We define ${V}$ inductively as follows:
% %\begin{itemize}
% %\item ${V}(\phi,X_\alpha) = \phi(X_\alpha)$
% %\item ${V}(\phi, p_\alpha) = I(p_\alpha)$
% %\item ${V}(\phi,s_{\alpha \typearrow \beta} t_\alpha)$ = $({V}(\phi,s_{\alpha \typearrow \beta}))({V}(\phi,t_\alpha))$
% %\end{itemize}
% \end{proof}
% \end{proposition}
%
\begin{lemma}
Let ${H}^{M}$ be a Henkin model for a \cjl\ model $M$. In ${H}^{M}$ we
have for all $s \in D_{i}$ and all $\bar{X}, \bar{Y}, \bar{Z} \in
D_{\proptype}$ (cf. the conditions \cjl\ models as stated on page~3):\footnote{In the
  proof we implicitly employ curring and uncurring, and we associate
  sets with their characteristic functions. This analogously applies to the remainder of this article.} 

\noindent
\begin{tabular}{ll}
 (av) & $ Iav_{i\typearrow\proptype} (s) \neq
        \emptyset$. \\
(pv1) & $ Iav_{i\typearrow\proptype}
        (s)\subseteq Ipv_{i\typearrow\proptype}(s)$. \\
(pv2) & $s \in Ipv_{i\typearrow\proptype}(s)$. \\
(ob1) & $\emptyset \notin
        Iob_{\proptype \typearrow \proptype \typearrow o}(\bar{X})$. \\
(ob2) & $\text{If } \bar{Y} \cap \bar{X} = \bar{Z} \cap \bar{X}
        \text{, then } (\bar{Y} \in Iob_{\proptype \typearrow
        \proptype \typearrow o}(\bar{X})  \text{ iff }  \bar{Z} \in
        Iob_{\proptype \typearrow \proptype \typearrow o}(\bar{X} )
        )$. \\
(ob3) & Let $ \bar{\beta}
        \subseteq Iob_{\proptype \typearrow \proptype \typearrow
        o}(\bar{X} )$ and $\bar{\beta} \neq \emptyset $.\\
        & If $ (\cap
        \bar{\beta} ) \cap \bar{X}\neq \emptyset$, where $\cap
        \bar{\beta} = \{s \in S\ | \text{ for all } \bar{Z}\in
        \bar{\beta} \text{ we have }  s\in \bar{Z} \}$, \\
        &  then $ ( \cap
        \bar{\beta}) \in Iob_{\proptype \typearrow \proptype
        \typearrow o}(\bar{X}) $.\\
(ob4) & $\text{If } \bar{Y}\subseteq \bar{X} \text{ and } \bar{Y} \in
        Iob_{\proptype \typearrow \proptype \typearrow o}(\bar{X})
        \text{ and } \bar{X} \subseteq \bar{Z},$\\
        &  $\text{then }(\bar{Z}
        \setminus \bar{X})\cup \bar{Y} \in Iob_{\proptype \typearrow
        \proptype \typearrow o}(\bar{Z})$. \\
(ob5) & $\text{If } \bar{Y} \subseteq \bar{X} \text{ and } \bar{Z}
        \in Iob_{\proptype \typearrow \proptype \typearrow o}(\bar{X}
        ) \text{ and } \bar{Y} \cap \bar{Z} \neq \emptyset,$ \\
        & $\text{then } \bar{Z} \in Iob_{\proptype \typearrow \proptype \typearrow
        o}(\bar{Y})$.
\end{tabular}
\end{lemma}
\begin{proof} Each statement follows by construction of
  ${H}^{M}$ for $M$. \\
(av): By definition of $ av $ for $ s \in S$ in $M$, $av(s) \neq \emptyset$; hence, there is $ u \in S $ such that $ u \in av(s) $. By definition of ${H}^{M}$, $Iav_{i\typearrow\proptype}(s,u) = T $, so $ u \in Iav_{i\typearrow\proptype}(s)$ and hence $ Iav_{i\typearrow\proptype}(s) \neq \emptyset$ in ${H}^{M}$.
\\
(pv1): By definition of $ av $ and $ pv $ for $ s \in S $ in $M$, $
av(s) \subseteq pv(s) $; hence, for every $ u \in av(s) $ we have $ u
\in pv(s) $. In ${H}^{M}$ this means, if
$Iav_{i\typearrow\proptype}(s,u) = T $, then $Ipv_{i\typearrow\proptype}(s,u) = T$. So, $Iav_{i\typearrow\proptype} (s)\subseteq Ipv_{i\typearrow\proptype}(s)$ in ${H}^{M}$.
\\
(pv2): This case is similar to (av).
\\
(ob1): By definition of $ ob $, we have  $\emptyset \notin ob(\bar{X})$; hence, in ${H}^{M}$, $Iob_{\proptype \typearrow \proptype \typearrow o}(\bar{X}, \emptyset) = F$, that is $ \emptyset \notin Iob_{\proptype \typearrow \proptype \typearrow o}(\bar{X}) $.
 \\
(ob2): Suppose $ \bar{Y} \cap \bar{X} = \bar{Z} \cap \bar{X} $. In $ M
$ we have $  \bar{Y} \in ob(\bar{X}) $ iff $ \bar{Z} \in ob(\bar{X})
$. By definition of ${H}^{M}$ we have $  Iob_{\proptype \typearrow \proptype \typearrow o}(\bar{X},\bar{Y}) = T $ iff $  Iob_{\proptype \typearrow \proptype \typearrow o}(\bar{X},\bar{Z}) = T $. Hence, $ \bar{Y} \in Iob_{\proptype \typearrow \proptype \typearrow o}(\bar{X}) $ iff $ \bar{Z} \in Iob_{\proptype \typearrow \proptype \typearrow o}(\bar{X}) $ in ${H}^{M}$.
 \\
 (ob3): Suppose  $ \bar{\beta} \subseteq Iob_{\proptype \typearrow \proptype \typearrow o}(\bar{X}) $ and $ \bar{\beta} \neq \emptyset $. If $ (\cap \bar{\beta} ) \cap \bar{X}\neq \emptyset$, by definition of $ ob $ in $M$ we have $ (\cap \bar{\beta} )\in ob(\bar{X}) $. Hence, in ${H}^{M}$, $Iob_{\proptype \typearrow \proptype \typearrow o}(\bar{X}, (\cap \bar{\beta} ) ) = T $ and then $  (\cap \bar{\beta} ) \in Iob_{\proptype \typearrow \proptype \typearrow o}(\bar{X})$.
 \\
 (ob4) and (ob5) are similar to (ob2). \qed
\end{proof}

\begin{lemma}
	Let ${H}^{M} = \langle \{D_\alpha\}_{\alpha \in {T}}, I
        \rangle$ be a Henkin model for a \cjl\ model $M$. We have $
        {H}^{M} \models^\text{HOL} \Sigma $ for all $\Sigma
        \in \{AV,PV1,PV2,OB1,...,OB5 \} $, where 
\begin{tabular}{lcl}
AV & is & $\forall W_{i} \exists V_{i}
        (av_{i\typearrow\proptype} W_{i} V_{i}) $ \\
PV1 & is & $\forall W_{i}  \forall V_{i}
        (av_{i\typearrow\proptype} W_{i} V_{i} \rightarrow
         pv_{i\typearrow\proptype} W_{i} V_{i})$ \\
PV2 & is & $\forall W_{i} (pv_{i\typearrow\proptype} W_{i} W_{i}) $ \\
OB1 & is & $ \forall X_{\proptype} \neg ob_{\proptype \typearrow
        \proptype \typearrow o} X_{\proptype} (\lambda
        X_{\proptype} \bot) $ \\
OB2 & is & $\forall X_{\proptype}Y_{\proptype}Z_{\proptype} (\,\,\,(\forall
        W_{i} ((Y_{\proptype} W_{i} \wedge X_{\proptype} W_{i})
        \longleftrightarrow (Z_{\proptype} W_{i} \wedge X_{\proptype}
        W_{i})))$ \\
   &       & \qquad \qquad \quad $\rightarrow 
		 (ob_{\proptype \typearrow \proptype \typearrow
  o} X_{\proptype} Y_{\proptype} \longleftrightarrow
        ob_{\proptype\typearrow \proptype \typearrow o} X_{\proptype}
        Z_{\proptype} ))$ \\
OB3 & is & $ \forall \beta_{\proptype \typearrow \proptype \typearrow
        o} \forall X_{\proptype}$ \\
  &    & $(\,\,\,((\forall Z_{\proptype}
        (\beta_{\proptype \typearrow \proptype \typearrow o}
        Z_{\proptype} \rightarrow ob_{\proptype \typearrow \proptype
        \typearrow o} X_{\proptype} Z_{\proptype})) \wedge \exists
        Z_{\proptype} (\beta_{\proptype \typearrow \proptype
        \typearrow o} Z_{\proptype}))$ \\
&   &    \quad  $\rightarrow (\,\,\, (\exists Y_{i} (((\lambda W_{i} \forall Z_{\proptype}
    (\beta_{\proptype \typearrow \proptype \typearrow o} Z_{\proptype}
    \rightarrow Z_{\proptype} W_{i}))\,Y_{i}) \wedge
  X_{\proptype} Y_{i}))$ \\
&   & \quad \qquad $\rightarrow
		ob_{\proptype \typearrow \proptype \typearrow o}
    X_{\proptype} (\lambda W_{i} \forall Z_{\proptype}
    (\beta_{\proptype \typearrow \proptype \typearrow o} Z_{\proptype}
    \rightarrow Z_{\proptype} W_{i}))))  $  \\
OB4 & is & $\forall X_{\proptype} Y_{\proptype} Z_{\proptype} $ \\
& & $(\,\,\, (\forall W_{i} (Y_{\proptype} W_{i} \rightarrow X_{\proptype} W_{i}) \wedge ob_{\proptype \typearrow \proptype \typearrow o}X_{\proptype}Y_{\proptype} \wedge \forall   X_{\proptype} (X_{\proptype}W_{i}\rightarrow Z_{\proptype}W_{i} )) $\\
&  &\quad $\rightarrow ob_{\proptype \typearrow \proptype \typearrow o}Z_{\proptype}(\lambda W_{i} ((Z_{\proptype}W_{i}\wedge \neg X_{\proptype}W_{i}) \vee Y_{\proptype}W_{i})))  $\\
OB5 & is & $\forall X_{\proptype} Y_{\proptype} Z_{\proptype}$ \\
& & $(\,\,\, (\forall W_{i} (Y_{\proptype}W_{i} \rightarrow X_{\proptype}W_{i}) \wedge ob_{\proptype\typearrow \proptype \typearrow o}X_{\proptype}Z_{\proptype}\wedge \exists
  W_{i} (Y_{\proptype}W_{i} \wedge Z_{\proptype}W_{i})) $ \\
& & \quad $\rightarrow ob_{\proptype \typearrow \proptype \typearrow o}Y_{\proptype}Z_{\proptype}) $\\
\end{tabular}
\end{lemma}

\begin{proof} We present detailed arguments for most cases. \\

\noindent AV: \\[-1.75em]
\begin{tabbing}
\qquad \= 
  For all $s \in D_{i}$: $I av_{i\typearrow\proptype} (s) \neq
  \emptyset$  \qquad (by Lemma 1 (av)) \\

$\Leftrightarrow$ 
   \> For all $s \in D_{i}$, there exists $u \in D_{i}$ such that $Iav_{i\typearrow\proptype}(s,u) = T$ \\

$\Leftrightarrow$ 
   \>  For all  assignments $g$, for all $s \in D_{i}$, there exists $u \in D_{i}$ such that  \\
   \> $\| av\, W\, V \|^{{H}^{M},g[s/W_{i}][u/V_{i}] }=T$ \\

$\Leftrightarrow$ 
  \> For all $g$, all $s \in D_{i}$ we have $\|\exists V (av\, W\, V) \|^{{H}^{M},g[s/W_{i}] }=T$ \\

$\Leftrightarrow$ 
  \> For all $g$ we have $\|\forall W \exists V (av\,W\,V)\|^{{H}^{M},g}=T$ \\

$\Leftrightarrow$ 
  \> ${H}^{M} \models^\text{HOL} AV $
\end{tabbing}

\noindent PV1: \\[-1.75em]
\begin{tabbing}
	\qquad \= 
	Given an arbitary assignment $g$, and arbitary $ s, u \in D_{i} $ such that \\
	\> $ \|av\, W\, V \|^{{H}^{M},g[s/W_\itype][u/V_\itype]}=T $\\
	
	$\Leftrightarrow$ 
	\>  $ Iav_{i\typearrow\proptype}(s,u)= T   $ \\
	
	$\Rightarrow$ 
	\> $Ipv_{i\typearrow\proptype}(s,u)= T  $ \qquad ($
        Iav_{i\typearrow\proptype} (s)\subseteq
        Ipv_{i\typearrow\proptype}(s) $, by Lemma 1 (pv1)) \\
	$\Leftrightarrow$ 
	\>  $ \|pv\, W\, V \|^{{H}^{M}, g [s/W_\itype][u/V_\itype] }=T  $
\end{tabbing}

\begin{tabbing}
	\qquad \= 
	
	Hence by definition of $\|.\|$, for all $g$, for all $ s, u \in D_{i}$ we have:\\
	
	\>  $  \|av\, W\, V \|^{{H}^{M},g[s/W_\itype][u/V_\itype]}=T $ implies $ \|pv\, W\, V \|^{{H}^{M},g[s/W_\itype][u/V_\itype]}=T $\\
	 
	$\Leftrightarrow$ 
	\>  For all $g$, all $ s, u \in D_{i} $ we have $ \| av\, W\, V \rightarrow pv\, W\, V\|^{{H}^{M},g[s/W_\itype][u/V_\itype]}=T $\\
	
	$\Leftrightarrow$ 
	\>  For all $ g $, all $ s \in D_{i}$ we have $ \| \forall V\ (av\, W\, V \rightarrow pv\, W\, V) \|^{{H}^{M},g[s/W_\itype]}=T $ \\
	
	$\Leftrightarrow$ 
	\> For all $ g $ we have $ \|\forall W\, \forall V\ (av\, W\, V  \rightarrow pv\, W\, V)) \|^{{H}^{M},g} = T $ \\
	$\Leftrightarrow$ 
	\>  $ {H}^{M} \models^\text{HOL} PV1 $
\end{tabbing}

\noindent PV2: \\[-1.75em]
\begin{tabbing} 
	\qquad \=  This case is analogous to AV. 
\end{tabbing}

\noindent OB1: \\[-1.75em]
\begin{tabbing}
	\qquad \= 
	 For all  $ \bar{X} \in D_{\proptype} $ : $ \emptyset \notin
         Iob_{\proptype \typearrow \proptype \typearrow o}(\bar{X}) $
         \qquad  (by Lemma 1 (ob1)) \\
	
	$\Leftrightarrow$ 
	\> For all  $g$, all $ \bar{X} \in D_{\proptype} $ we have  $ \|\neg ob\, X\, (\lambda X. \bot) \|^{{H}^{M}, g [\bar{X}/X_\proptype] }=T $  \\
	
	$\Leftrightarrow$ 
	\> For all $ g $ we have $ \|\forall X\, \neg (ob\, X\,
        (\lambda X_{\proptype} \bot)) \|^{{H}^{M},g[\bar{X}/X_\proptype]}=T $ \\
	
	$\Leftrightarrow$ 
	\> $ {H}^{M} \models^\text{HOL} OB1 $
\end{tabbing}

\noindent OB2: \\[-1.75em]
\begin{tabbing}
	\qquad \= 
    Given an arbitary assignment $ g $, and arbitary $ \bar{X}, \bar{Y}, \bar{Z} \in D_{\proptype} $ such that\\
	\> $ \|\forall W ((Y\, W  \wedge X\, W)\longleftrightarrow (Z\, W \wedge X\, W )) \|^{{H}^{M},g[\bar{X}/X_\proptype][\bar{Y}/Y_\proptype][\bar{Z}/Z_\proptype]} = T  $\\
	$\Leftrightarrow$ 
	\> For all $ s \in D_{i} $ we have\\
	\> $ \|(Y\, W \wedge X\, W) \longleftrightarrow (Z\, W \wedge X\, W)\|^{{H}^{M}, g [\bar{X}/X_\proptype] [\bar{Y}/Y_\proptype] [\bar{Z}/Z_\proptype][s/W_i] } = T  $\\
	
	$\Leftrightarrow$ 
	\> For all $ s \in D_{i} $ we have \\
	\> $ \|Y\, W \wedge X\, W \|^{{H}^{M},g[\bar{X}/X_\proptype][\bar{Y}/Y_\proptype] [\bar{Z}/Z_\proptype][s/W_i]} = T $ \quad  iff \\
	\> $ \| Z\, W \wedge X\, W\|^{{H}^{M},g[\bar{X}/X_\proptype][\bar{Y}/Y_\proptype] [\bar{Z}/Z_\proptype][s/W_i]} = T $\\
	
	$\Leftrightarrow$ 
	\> For all $ s \in D_{i} $ we have  $ s \in \bar{Y} \cap \bar{X} $ iff $ s \in  \bar{Z} \cap \bar{X} $ \\
	
	$\Leftrightarrow$ 
	\> $ \bar{Y} \cap \bar{X} = \bar{Z} \cap \bar{X} $ \\
	
	$\Rightarrow$ 
	\> $ Iob_{\proptype \typearrow \proptype \typearrow
          o}(\bar{X},\bar{Y}) = T $  iff  $ Iob_{\proptype \typearrow
          \proptype \typearrow o}(\bar{X},\bar{Z}) = T $ \quad (by
        Lemma 1 (ob2)) \\
	
	$\Leftrightarrow$ 
	\> $ \|ob\, X\, Y)\|^{{H}^{M}, g [\bar{X}/X_\proptype]
          [\bar{Y}/Y_\proptype] [\bar{Z}/Z_\proptype] } = T $  iff   \\
	\> $\|ob\, X\, Z \|^{{H}^{M}, g [\bar{X}/X_\proptype] [\bar{Y}/Y_\proptype] [\bar{Z}/Z_\proptype] } = T $ \\
	
	$\Leftrightarrow$ 
	\> $ \|ob\, X\, Y \longleftrightarrow ob\, X\, Z \|^{{H}^{M}, g [\bar{X}/X_\proptype] [\bar{Y}/Y_\proptype] [\bar{Z}/Z_\proptype] } = T $  
	\\
	\\
	Hence, by definition of $\|.\|$, for all $ g $, for all  $ \bar{X}, \bar{Y}, \bar{Z} \in D_{\proptype} $  we have:\\ 
	
	\> $\|( \forall W (\, ((Y\, W \wedge X\, W) \longleftrightarrow (Z\, W \wedge X\, W))  $ \\
	\>$ \qquad \quad \rightarrow (ob\, X\, Y \longleftrightarrow ob \, X\, Z))\|^{{H}^{M},g [\bar{X}/X_\proptype] [\bar{Y}/Y_\proptype] [\bar{Z}/Z_\proptype] }= T  $\\
	
	$\Leftrightarrow$ 
	\>  For all $g$ we have\\
	\>$\|\forall X Y Z (\forall W (\,((Y\, W \wedge X\, W ) \longleftrightarrow (Z\, W \wedge X\, W)) $ \\
	\> $\qquad \quad \qquad \quad \rightarrow (ob\, X\, Y \longleftrightarrow ob\, X\, Z )) \|^{{H}^{M}, g } = T$\\
	
	$\Leftrightarrow$ 
	\>  $ {H}^{M} \models^\text{HOL} OB2 $
\end{tabbing}

\vfill

\noindent OB3: \\[-1.75em]
\begin{tabbing}
	\qquad \= 
	 Given an arbitary assignment $g$, and arbitary $ \bar{\beta}
         \in D_{\proptype \typearrow o}, \bar{X} \in D_{\proptype} $
         \\
        \> such that\\
	\> $ \| \forall Z  (\beta\, Z \rightarrow ob\, X\, Z) \|^{{H}^{M}, g  [\bar{\beta} /\beta_{\proptype \typearrow o}] [\bar{X}/X_\proptype] }= T $  \quad  and  \\  
	\> $ \| \exists Z (\beta\, Z) \|^{{H}^{M},g[\bar{\beta} /\beta_{\proptype \typearrow o}] }= T $ \quad  and \\
	\> $ \| \exists Y (((\lambda W \forall Z (\beta\, Z
        \rightarrow Z\, W))\, Y) \wedge X\, Y) \|^{{H}^{M},g[\bar{\beta} /\beta_{\proptype \typearrow o}] [\bar{X}/X_\proptype] }= T $ \\
	$\Leftrightarrow$ 
	\> For all  $ \bar{Z} \in D_{\proptype} $ we have\\
	\>  \quad $ \| \beta\, Z \|^{{H}^{M},g[\bar{\beta}
          /\beta_{\proptype \typearrow o}][\bar{X}/X_\proptype]
          [\bar{Z}/Z_\proptype] }= T $ implies  \\  
        \> \quad $ \|ob\, X\, Z \|^{{H}^{M}, g [\bar{\beta}
          /\beta_{\proptype \typearrow
            o}][\bar{X}/X_\proptype][\bar{Z}/Z_\proptype]}= T $  \quad
        and  \\
	\> there exists $ \bar{Z} \in D_{\proptype}  $ such that $ \| \beta \, Z  \|^{{H}^{M},g[\bar{\beta} /\beta_{\proptype \typearrow o}] [\bar{Z}/Z_\proptype] }= T $  \quad and \\
	\>  there exists  $s \in D_i$ such that \\
	\> $ \|(\lambda W \forall Z (\beta\, Z \rightarrow Z\, W))\, Y \wedge X\, Y \|^{{H}^{M},g[\bar{\beta} /\beta_{\proptype \typearrow o}] [\bar{X}/X_\proptype][s/Y_{i}]}= T $ \\
	
	$\Leftrightarrow$ 
	\> For all $ \bar{Z} \in D_{\proptype} $ we have $\bar{Z} \in \beta$ implies  $\bar{Z} \in Iob_{\proptype \typearrow \proptype \typearrow o}(\bar{X}) $ \quad   and \\ 
	\> there exists  $ \bar{Z} \in D_{\proptype}$ such that $ \bar{Z} \in \bar{\beta} $  \quad  and   \\ 
	\> there exists  $ s \in D_i$ such that $ s \in
        \cap\bar{\beta} $ and $ s\in \bar{X} $ \quad (\textbf{see *})\\
	
	$\Leftrightarrow$ 
	\> $ \bar{\beta} \subseteq Iob_{\proptype \typearrow \proptype \typearrow o}(\bar{X}) $ and $ \bar{\beta} \neq \emptyset $  and  $ (\cap\bar{\beta})\cap \bar{X}\neq\emptyset $ \\
	
	$\Rightarrow$ 
	\> $ Iob_{\proptype \typearrow \proptype \typearrow
          o}(\bar{X}, ( \cap \bar{\beta})) = T $ \quad (by Lemma 1 (ob3)) \\
	
	$\Leftrightarrow$ 
	\> $ \|ob\, X\, (\lambda W  \forall Z (\beta\, Z  \rightarrow  Z\, W ) )\|^{{H}^{M},g [\bar{\beta} /\beta_{\proptype\typearrow o}][\bar{X}/X_\proptype] }= T  $
	\\
	\\ 
    \> Hence by definition of $\|.\|$, for all $g$, all $ \bar{\beta} \in D_{\proptype \typearrow o} $, all $\bar{X} \in D_{\proptype} $ we have:\\ 
	
	\> $ \| ((\forall Z (\beta\, Z  \rightarrow ob\, X\, Z  ) )
	\wedge (\exists Z (\beta\, Z  )))$ \\
	\> $\rightarrow ( (\exists Y (((\lambda W \forall Z (\beta\, Z \rightarrow Z\, W))Y) \wedge X\, Y) ) $ \\
	\> $\rightarrow ob\, X\, (\lambda W  \forall Z (\beta\, Z  \rightarrow  Z\, W ) ) ) \|^{{H}^{M},g[\bar{\beta} /\beta_{\proptype\typearrow o}][\bar{X}/X_\proptype] } = T $ \\
	
	% $\Leftrightarrow$ 
	% \>  For all $g$, all $ \bar{\beta} \in D_{\proptype \typearrow o} $ we have \\
	% \> $ \| \forall X ( ((\forall Z (\beta\, Z  \rightarrow ob\, X\, Z  ) )
	% \wedge (\exists Z (\beta\, Z  )))$ \\
	% \> $\rightarrow ( (\exists Y (((\lambda W \forall Z (\beta\, Z \rightarrow Z\, W))Y) \wedge X\, Y) ) $ \\
	% \> $\rightarrow ob\, X\, (\lambda W  \forall Z (\beta\, Z  \rightarrow  Z\, W ) ) ) ) \|^{{H}^{M},g [\bar{\beta} /\beta_{\proptype\typearrow o}]} = T $ \\
	
	$\Leftrightarrow$ 
	\>  For all $g$, we have \\
	\> $ \| \forall \beta  \forall X ( ((\forall Z (\beta\, Z  \rightarrow ob\, X\, Z  ) )
	\wedge (\exists Z (\beta\, Z  )))$ \\
	\> $\rightarrow ( (\exists Y (((\lambda W \forall Z (\beta\, Z \rightarrow Z\, W))Y) \wedge X\, Y) ) $ \\
	\> $\rightarrow ob\, X\, (\lambda W  \forall Z (\beta\, Z  \rightarrow  Z\, W ) ) ) ) \|^{{H}^{M},g} = T $ \\

	$\Leftrightarrow$ 
	\>  $ {H}^{M} \models^\text{HOL} OB3 $
\end{tabbing}

\vfill

\fbox{\begin{minipage}{.93\textwidth}
		\textbf{Justification *:} % $ \lambda W_{i} \forall
                % Z_{\proptype} (\beta_{\proptype \typearrow o}
                % Z_{\proptype}\rightarrow Z_{\proptype} W_{i}) $ is a
                % term of type $\proptype  $.
                \sloppy
                By definition of $\|.\|$,
                $\|\lambda W_{i} \forall Z_{\proptype}
                (\beta_{\proptype \typearrow o}
                Z_{\proptype}\rightarrow Z_{\proptype} W_{i})
                \|^{{H}^{M},g[\bar{\beta} /\beta_{\proptype \typearrow
                    o}] [\bar{X}/X_\proptype][s/Y_{i}]}$ is denoting
                the function $ f $ from $D_{i}$ to $D_{o}$ such that for all $ d
                \in D_{i} $, $
                f(d) = \|\forall Z_{\proptype} (\beta_{\proptype
                  \typearrow o} Z_{\proptype}\rightarrow Z_{\proptype}
                W_{i}) \|^{{H}^{M},g[\bar{\beta} /\beta_{\proptype
                    \typearrow o}]
                  [\bar{X}/X_\proptype][s/Y_{i}][d/W_i]}$. 
                By definition~of $\|.\|$,
                $\|\forall Z_{\proptype} (\beta_{\proptype \typearrow
                  o} Z_{\proptype}\rightarrow Z_{\proptype} W_{i})
                \|^{{H}^{M},g[\bar{\beta} /\beta_{\proptype \typearrow
                    o}] [\bar{X}/X_\proptype][s/Y_{i}][d/W_i]} = T$
                iff
                $\text{ for all } \bar{Z} \in \bar{\beta} \text{ we
                  have } d \in \bar{Z} $.
                Thus,  $f$ is the characteristic function of the set $\cap\bar{\beta}$. By the
                Denotatpflicht, which is obeyed in $H^M$, we know that
                $f (= \cap\bar{\beta}) \in D_\tau$.

% By def.~of $\|.\|$, for all variable assignments $ g $, for all $\bar{\beta} \in \beta_{\proptype \typearrow o} $, for all $\bar{Z} \in D_{\proptype} $, for all $ s \in D_{i} $ we have:
% 		\begin{center}
% 			$  \|\lambda W_{i} \forall Z_{\proptype} (\beta_{\proptype \typearrow o} Z_{\proptype}\rightarrow Z_{\proptype} W_{i})\|^{{H}^{M}, g [\bar{\beta} /\beta_{\proptype \typearrow o}]  [\bar{Z}/Z_\proptype] [s/W_i] } = T $ iff\\
% 			$ \text{ For all } \bar{Z} \in \beta \text{ we have } s \in \bar{Z} $   
% 		\end{center}
% 		Hence by definition of $\|.\|$, $ f(s) = T $ iff $ s \in \cap\bar{\beta} $ for all $ s \in D_{i} $.
% 		\\
% 		So, $\lambda W_{i} \forall Z_{\proptype} (\beta_{\proptype \typearrow o} Z_{\proptype}\rightarrow Z_{\proptype} W_{i}) $ represents $ \cap\bar{\beta} $. 
% 		\\
\end{minipage}}
\\[.5em]

\noindent OB4: \\[-1.75em]
\begin{tabbing}
	\qquad \= 
	Given an arbitary assignment $ g $, and arbitary  $ \bar{X}, \bar{Y}, \bar{Z} \in D_{\proptype} $ such that\\
	\> $ \|\forall W (Y\, W \rightarrow X\, W) \wedge ob\, X\, Y
        \wedge $ \\
         \> $\;\; \forall W ( X\, W \rightarrow Z\, W) \|^{{H}^{M}, g [\bar{X}/X_\proptype][\bar{Y}/Y_\proptype] [\bar{Z}/Z_\proptype] } = T  $\\
	
	$\Leftrightarrow$
	\>  $\|\forall W (Y\, W \rightarrow X\, W) \|^{{H}^{M}, g
          [\bar{X}/X_\proptype] [\bar{Y}/Y_\proptype] [\bar{Z}/Z_\proptype] } = T $ \quad  and   \\
	\> $ \|ob\, X\, Y\|^{{H}^{M}, g [\bar{X}/X_\proptype] [\bar{Y}/Y_\proptype] [\bar{Z}/Z_\proptype] } = T $ \quad  and \\
	\> $ \|\forall W ( X\, W \rightarrow Z\, W) \|^{{H}^{M}, g [\bar{X}/X_\proptype] [\bar{Y}/Y_\proptype] [\bar{Z}/Z_\proptype] } = T  $\\
	
    %     $\Leftrightarrow$
    %     \> For all $ s \in D_{i} $ we have\\
    %     \> $ \|Y\, W\|^{{H}^{M},g[\bar{X}/X_\proptype]
    %       [\bar{Y}/Y_\proptype][\bar{Z}/Z_\proptype][s/W_i]} = T $  implies \\
    %     \> $ \| X\, W \|^{{H}^{M}, g [\bar{X}/X_\proptype][\bar{Y}/Y_\proptype][\bar{Z}/Z_\proptype][s/W_i]} = T $ \quad and \\
    %     \> $ \|ob\, X\, Y\|^{{H}^{M},g[\bar{X}/X_\proptype][\bar{Y}/Y_\proptype][\bar{Z}/Z_\proptype] } = T $ \quad and \\
    % \> $ \| X\, W \|^{{H}^{M}, g [\bar{X}/X_\proptype]
    %   [\bar{Y}/Y_\proptype][\bar{Z}/Z_\proptype][s/W_i]} = T  $ implies \\
    %     \> $ \| Z\, W \|^{{H}^{M},g[\bar{X}/X_\proptype] [\bar{Y}/Y_\proptype][\bar{Z}/Z_\proptype][s/W_i]} = T  $\\
 
	$\Leftrightarrow$
	\> For all $ s \in D_{i} $ we have \\
	\> ($ s \in \bar{Y} $ implies $ s \in \bar{X} $) and  $ \bar{Y} \in Iob_{\proptype \typearrow \proptype \typearrow o}(\bar{X}) $ and ($ s \in \bar{X} $ implies $ s \in \bar{Z} $)  \\
	
	$\Leftrightarrow$
	\> $ \bar{Y} \subseteq \bar{X} $  and $ \bar{Y} \in Iob_{\proptype \typearrow \proptype \typearrow o}(\bar{X}) $ and  $ \bar{X} \subseteq \bar{Z} $ \\
	
	$\Rightarrow$
	\> $ (\bar{Z} \setminus \bar{X})\cup \bar{Y} \in
        Iob_{\proptype \typearrow \proptype \typearrow o}(\bar{Z}) $
        \quad (by Lemma 1 (ob4)) \\
	
	$\Leftrightarrow$
	\> $ \|ob\, Z\, (\lambda W (( Z\, W \wedge\neg X\, W)\vee Y\,
        W))\|^{{H}^{M}, g [\bar{X}/X_\proptype][\bar{Y}/Y_\proptype]
          [\bar{Z}/Z_\proptype] } = T $  (\textbf{see **}) \\
	\\
	\> Hence by definition of $\|.\|$  for all $ g $, all $\bar{X}, \bar{Y}, \bar{Z} \in D_{\proptype} $ we have  \\
	\> $ \|(\forall W (Y\, W \rightarrow X\, W) \wedge ob\, X\, Y \wedge \forall W ( X\, W \rightarrow Z\, W)  ) $ \\
	 \> $\rightarrow ob\, Z\, (\lambda W (( Z\, W \wedge\neg X\, W)\vee Y\, W))\|^{{H}^{M},g[\bar{X}/X_\proptype][\bar{Y}/Y_\proptype] [\bar{Z}/Z_\proptype] } = T$ \\
	$\Leftrightarrow$ 
	\> For all  $g$ we have \\
	\> $ \|\forall X Y Z ( (\forall W (Y\, W \rightarrow X\, W) \wedge ob\, X\, Y \wedge \forall W ( X\, W \rightarrow Z\, W)  )$ \\
	 \>$\rightarrow ob\, Z\, (\lambda W (( Z\, W \wedge\neg X\, W)\vee Y\, W)))\|^{{H}^{M},g  } = T$\\
	
	$\Leftrightarrow$
	\>  $ {H}^{M} \models^\text{HOL} OB4 $
	
\end{tabbing}

\fbox{\begin{minipage}{.93\textwidth} \textbf{Justification **:}
    Similar to justification *, we can convince ourselves that
    $\|\lambda W (( Z\, W \wedge\neg X\, W)\vee Y\, W)\|^{{H}^{M}, g
      [\bar{X}/X_\proptype] [\bar{Y}/Y_\proptype][\bar{Z}/Z_\proptype]
      [\bar{Z}/Z_\proptype]}$
    is denoting the characteristic function $f$ of the set
    $(\bar{Z} \setminus \bar{X})\cup \bar{Y}$. By the Denotatpflicht,
    which is obeyed in $H^M$, we know that $f (= (\bar{Z} \setminus
    \bar{X})\cup \bar{Y}) \in D_\tau$.
% The expression $ \lambda W_{i} (Z_{\proptype} W_{i} \wedge \neg X_{\proptype}W_{i} )$ is an expression of type $\proptype  $.  By definition of $\|.\|$, $ \|\lambda W_{i} (Z_{\proptype} W_{i} \wedge \neg X_{\proptype}W_{i}  \|^{{H}^{M}, g}$ is denoting the function $ f $ from $D_{i}$ to $D_{o}$ such that $ f(s) = \| \lambda W_{i} (Z_{\proptype} W_{i} \wedge \neg X_{\proptype}W_{i}) \|^{{H}^{M}, g [s/W_i] }$ for all $ s \in D_{i} $. Hence by definition of $\|.\|$, for all variable assignments $ g $, for $\bar{X}, \bar{Y}, \bar{Z} \in D_{\proptype} $ and for all $ s \in D_{i} $ we have:
% 		\begin{center}
% 			$ \| \lambda W_{i} (Z_{\proptype} W_{i} \wedge \neg X_{\proptype}W_{i} ) \|^{{H}^{M}, g [s/W_i] [\bar{X}/X_\proptype] [\bar{Y}/Y_\proptype] [\bar{Z}/Z_\proptype]  } = T $  iff  $ s \in (\bar{Z} \setminus \bar{X})\cup \bar{Y} $.
% 		\end{center}
% 		Hence, the expression $\lambda W_{i} (Z_{\proptype} W_{i} \wedge \neg X_{\proptype}W_{i}) $ represents $ (\bar{Z} \setminus \bar{X})\cup \bar{Y} $.
\end{minipage}}
\\
\\

\noindent OB5:  \\[-1.75em]
\begin{tabbing} 
	\qquad \=  This case is analogous to OB4.  
\end{tabbing} \qed
\end{proof}
\begin{lemma}\label{lemma:3} \sloppy
	Let ${H}^{M}$ be a Henkin model for a \cjl\ model $M$. For all \cjl\ formulas $\delta$, arbitrary variable assignments $g$ and worlds $s$ it holds: $M,s \models \delta \text{\ if and only if\ } \| \lfloor \delta \rfloor\, S_\itype\|^{{H}^{M}, g 
	[s/S_\itype]}=T$.
\end{lemma}
\begin{proof}
The proof of the lemma  is by induction on the structure of~$\delta$. 

\noindent In the base case we have $\delta = p^{j}$ for some $p^{j}\in P$: 
\begin{tabbing}
	\qquad \= $\| \lfloor p^{j} \rfloor S\|^{{H}^{M}, g [s/S_\itype] }=T$ \\
	$\Leftrightarrow$ 
	\> $ \| p^{j}_\proptype S\|^{{H}^{M}, g [s/S_\itype] } = T  $\\
	
	$\Leftrightarrow$ 
	\>  $ I p^{j}_\proptype (s) = T $  \\
	
	$\Leftrightarrow$ 
	\>  $  s \in V(p^{j}) $ \quad (by definition of $ {H}^{M} $)\\
	
	$\Leftrightarrow$ 
	\> $ M,s \vDash p^{j} $ 
\end{tabbing}

For proving the inductive cases we apply the induction hypothesis,
which is formulated as follows: For all $\delta'$ that are structurally smaller than
$\delta$, for all assignment $g$ and all $s$ we have 
			$ \| \lfloor \delta' \rfloor S \|^{{H}^{M}, g [s/S_\itype] } = T
			$ if and only if 
			$ M,s \vDash \delta' $. 

We consider each inductive case in turn: \\
\noindent $\delta = \neg \varphi$: \\[-1.5em]
\begin{tabbing}
	\qquad \= 
	$ \| \lfloor \neg \varphi\rfloor S\|^{{H}^{M}, g[s/S_\itype] } = T  $\\
	$\Leftrightarrow$ 
	\> $ \|(\neg_{\proptype \typearrow \proptype} \lfloor
        \varphi\rfloor) S\|^{{H}^{M}, g[s/S_\itype]} = T$ \\
	$\Leftrightarrow$ 
	\>  $ \|\neg (\lfloor \varphi\rfloor
        S)\|^{{H}^{M},g[s/S_\itype]} = T \qquad (\text{since }(\neg_{\proptype \typearrow \proptype} \lfloor
        \varphi\rfloor) S =_{\beta\eta}  \neg (\lfloor \varphi\rfloor
        S))$\\
	
	$\Leftrightarrow$ 
	\>  $ \| \lfloor \varphi\rfloor S\|^{{H}^{M}, g[s/S_\itype]} = F $ \\
	
	$\Leftrightarrow$ 
	\> $ M,s \nvDash \varphi $ \quad (by induction hypothesis)\\
	
	$\Leftrightarrow$ 
	\> $ M,s \vDash \neg \varphi $ 
\end{tabbing}	

\noindent $\delta = \varphi\vee\psi$: \\[-1.5em]
\begin{tabbing}
	\qquad \= 
	$ \| \lfloor \varphi\vee\psi \rfloor S\|^{{H}^{M},g[s/S_\itype] } = T $\\
	
	$\Leftrightarrow$ 
	\> $ \|(\lfloor \varphi\rfloor \vee_{\proptype \typearrow \proptype \typearrow \proptype} \lfloor \psi \rfloor)  S\|^{{H}^{M}, g[s/S_\itype]} = T  $\\
	
	$\Leftrightarrow$ 
	\>  $ \|(\lfloor \varphi \rfloor S ) \vee ( \lfloor \psi
        \rfloor  S)\|^{{H}^{M},g[s/S_\itype]} = T$ \\
        \> \qquad \qquad \qquad \qquad $(\text{since } (\lfloor \varphi\rfloor \vee_{\proptype \typearrow \proptype \typearrow \proptype} \lfloor \psi \rfloor) S =_{\beta\eta}  ((\lfloor \varphi \rfloor S) \vee ( \lfloor \psi \rfloor S ))) $  \\
	
	$\Leftrightarrow$ 
	\>  $ \|\lfloor \varphi \rfloor S\|^{{H}^{M}, g[s/S_\itype]} = T $ or 
$ 	\| \lfloor \psi \rfloor  S)\|^{{H}^{M},g[s/S_\itype]} = T $  \\
	
	$\Leftrightarrow$ 
	\> $  M,s \vDash \varphi $ or $ M,s \vDash \psi $ \quad (by induction hypothesis)\\
	
	$\Leftrightarrow$ 
	\> $ M,s \vDash \varphi  \vee \psi $ 
\end{tabbing}

\noindent $\delta = \Box \varphi$:  \\[-1.5em]
\begin{tabbing}
	\qquad \= 
	$ \| \lfloor \Box \varphi \rfloor S\|^{{H}^{M}, g  [s/S_\itype] } = T $\\
	
	$\Leftrightarrow$ 
	\> $ \| ( \lambda X \forall Y (\lfloor \varphi \rfloor  Y ) ) S \|^{{H}^{M},g[s/S_\itype]} = T$\\
	
	$\Leftrightarrow$ 
	\> For all  $ a \in D_{i} $ we have $ \| \lfloor\varphi \rfloor Y\|^{{H}^{M},g[s/S_\itype][a/Y_\itype]} = T $ \\

	$\Leftrightarrow$ 
	\> For all  $ a \in D_{i} $ we have $ \| \lfloor\varphi
        \rfloor Y\|^{{H}^{M},g[a/Y_\itype]} = T $ \quad ($S\notin free(\lfloor\varphi
        \rfloor)$)\\

	$\Leftrightarrow$ 
	\> For all $a \in S$ we have $ M,a\models\varphi $ \quad (by induction hypothesis) \\
	
	$\Leftrightarrow$ 
	\> $M,s\models\Box\varphi $ 
\end{tabbing}	
	
\noindent $\delta = \Box_{a} \varphi$: \\[-1.5em]
\begin{tabbing}
	\qquad \= 
	$ \| \lfloor \Box_{a} \varphi \rfloor S\|^{{H}^{M},g[s/S_\itype]} = T $\\
	
	$\Leftrightarrow$ 
	\> $ \| \left(\lambda X \forall Y (\neg av\, X\,  Y  \vee \lfloor\varphi \rfloor  Y ) \right) S \|^{{H}^{M},g[s/S_\itype]} = T$\\
	
	$\Leftrightarrow$ 
	\> For all $  a \in D_{i} $ we have $ \|  \neg av\, S\, Y \vee \lfloor\varphi \rfloor Y\|^{{H}^{M},g[s/S_\itype][a/Y_\itype]} = T $ \\
	
	$\Leftrightarrow$ 
	\> For all $ a \in D_{i} $ we have $ \| av\, S\, Y
        \|^{{H}^{M}, g  [s/S][a/Y] } = F $ or \\
        \>  $ \| \lfloor\varphi \rfloor Y\|^{{H}^{M},g[s/S_\itype][a/Y_\itype]} = T $ \\
	
	$\Leftrightarrow$ 
	\> For all  $ a \in D_{i} $ we have $
        Iav_{i\typearrow\proptype}(s,a) = F $ or \\
        \> $ \| \lfloor\varphi \rfloor
        Y\|^{{H}^{M},g[a/Y_\itype]} = T$ \qquad ($S\notin
        free(\lfloor\varphi \rfloor)$)\\
	
	$\Leftrightarrow$ 
	\> For all  $ a \in S $ we have $ a\not\in av(s) $ or \\
        \> $ M,a\models\varphi $ \qquad (by induction hypothesis) \\
	
	$\Leftrightarrow$ 
	\> $ M,s\models\Box_{a} \varphi $
	
\end{tabbing}	

\noindent $\delta = \Box_{p} \varphi$. \\[-1.5em]
\begin{tabbing}
	\qquad \=  The argument is analogous to $\delta = \Box_{a} \varphi$. 
\end{tabbing}	

\noindent $\delta =\bigcirc (\psi/\varphi)$: \\[-1.5em]
\begin{tabbing}
	\qquad \= 
	$ \| \lfloor\bigcirc (\psi/\varphi)  \rfloor S\|^{{H}^{M},g[s/S_\itype]} = T $\\
	
	$\Leftrightarrow$ 
	\>  $ \|  (\lambda X  (ob \lfloor \psi  \rfloor   \lfloor \varphi  \rfloor )) S \|^{{H}^{M},g[s/S_\itype]} = T $\\
	
	$\Leftrightarrow$ 
	\>   $ \| ob  \lfloor \psi \rfloor  \lfloor \varphi  \rfloor   \|^{{H}^{M},g[s/S_\itype]} = T $ \\
	
	$\Leftrightarrow$ 
	\>  $  Iob_{\proptype \typearrow \proptype \typearrow o} ( \| \lfloor \psi  \rfloor   \|^{{H}^{M},g[s/S_\itype]}) (\| \lfloor \varphi  \rfloor\|^{{H}^{M},g[s/S_\itype]}) = T $ \\
	
	$\Leftrightarrow$ 
	\> $ \| \lfloor \varphi  \rfloor \|^{{H}^{M}, g[s/S_\itype]} \in Iob_{\proptype \typearrow \proptype \typearrow o}(\| \lfloor\psi\rfloor \|^{{H}^{M},g[s/S_\itype]}) $\\
	
	$\Leftrightarrow$ 
	\> $ V(\varphi) \in Iob_{\proptype \typearrow \proptype \typearrow o}(V(\psi))$ \qquad
	(\textbf{see ***})  \\
	
	$\Leftrightarrow$ 
	\> $V(\varphi) \in ob(V(\psi)) $\\
	
	$\Leftrightarrow$ 
	\> $M,s\models\bigcirc (\psi/\varphi)$
	
\end{tabbing}

\fbox{\begin{minipage}{.93\textwidth}
			\textbf{Justification ***:} We need to show
                        that 
			$\| \lfloor \varphi \rfloor \|^{{H}^{M}, g [s/S_\itype] }$ is identified with
			$ V(\varphi) = \{s \in S\, |\, M, s
                          \models \varphi\}$ (analogous for $\psi$). By induction hypothesis, for
			all assignment $g$ and world $s$, we have
			$ \| \lfloor \varphi \rfloor S \|^{{H}^{M}, g [s/S_\itype] } = T $ if and only if
			$ M,s \vDash \varphi $.  We expand the details
                        of this equivalence. For  all assignment
			$g$ and all worlds $s\in D_i$ we have
			\begin{tabbing}
			\qquad  \= $ s \in \|
                        \lfloor \varphi \rfloor
                        \|^{{H}^{M},g[s/S_\itype]} $  \quad
                        ({charact.~functions are associated with sets})\\
			$\Leftrightarrow$ \> $\| \lfloor \varphi \rfloor \|^{{H}^{M},g[s/S_\itype]}(s) = T$ \\
			$\Leftrightarrow$ \> $\|\lfloor \varphi \rfloor
                        \|^{{H}^{M}, g [s/S_\itype]}  (\|S \|^{{H}, g
                          [s/S_\itype]} ) = T $ \\
			$\Leftrightarrow$ \> $\| \lfloor \varphi \rfloor S \|^{{H}^{M}, g [s/S_\itype]} = T$ \\
			$\Leftrightarrow$ \> $M,s \vDash \varphi$
                        \qquad (induction hypothesis)\\
			$\Leftrightarrow$ \> $s\in V(\varphi) $
			\end{tabbing}
			Hence, $s \in \| \lfloor \varphi \rfloor \|^{{H}^{M},g[s/S_\itype]}$ if and only
			if  $s\in V(\varphi)$. By extensionality we thus know that $\| \lfloor
			\varphi \rfloor \|^{{H}^{M},g[s/S_\itype]} =
                        V(\varphi)$. Moreover, since $H^M$ obeys the
                        Denotatpflicht we know that $V(\varphi) \in D_\tau$.
\end{minipage}}
\\[1em]

\noindent $\delta = \bigcirc_{a} (\varphi)$: \\[-1.5em]
\begin{tabbing}
	\qquad \= 
	$ \| \lfloor \bigcirc_{a} (\varphi) \rfloor S \|^{{H}^{M},g[s/S_\itype]} = T  $\\
	
	$\Leftrightarrow$ 
	\>  $ \| (\lambda X ( ob\, (av\, X) \lfloor \varphi  \rfloor
        \wedge \exists Y (av\,X \, Y \wedge \neg (\lfloor \varphi
        \rfloor Y) ) ) S \|^{{H}^{M}, g [s/S_\itype] } = T  $\\

	$\Leftrightarrow$ 
	\>  $ \| ob\, (av\, S) \lfloor \varphi  \rfloor \wedge \exists Y (av\,S\, Y \wedge \neg (\lfloor \varphi  \rfloor Y) ) \|^{{H}^{M}, g [s/S_\itype] } = T  $\\
	
	$\Leftrightarrow$ 
	\> $  \| ob\, (av\, S) \lfloor \varphi  \rfloor
        \|^{{H}^{M}, g [s/S_\itype] } = T $ \quad and  \\
        \> $ \| \exists Y (av\,S \, Y \wedge \neg (\lfloor \varphi  \rfloor Y) )\|^{{H}^{M},g[s/S_\itype]} = T $ \\
	
	$\Leftrightarrow$ 
	\> $ \| ob\, (av\, S) \lfloor \varphi  \rfloor \|^{{H}^{M}, g [s/S_\itype] } = T $ \quad and \\
	\> there exists $ a \in D_{i} $ such that $ \| av\,S \, Y \wedge \neg (\lfloor \varphi  \rfloor Y) \|^{{H}^{M},g[s/S_\itype][a/Y_\itype]} = T $\\
	
	$\Leftrightarrow$ 
	\> $ Iob_{\proptype \typearrow \proptype \typearrow o} (\|av\,
        S\|^{{H}^{M}, g [s/S_\itype] })(\|
        \lfloor \varphi \|^{{H}^{M},g[s/S_\itype]}) = T  $  \quad and \\
	 \> there exists $ a \in D_{i} $ such that \\
         \> $ \|av\,X \, Y  \|^{{H}^{M},g[s/S_\itype][a/Y_\itype]} = T
         $ and $ \| \lfloor \varphi  \rfloor Y
         \|^{{H}^{M},g[s/S_\itype][a/Y_\itype]} = F $  \\

	$\Leftrightarrow$ 
	\> $ \|
        \lfloor \varphi \|^{{H}^{M},g[s/S_\itype]}
        \in  Iob_{\proptype \typearrow \proptype \typearrow o} (\|av\,
        S\|^{{H}^{M}, g [s/S_\itype] })$  \quad and \\
	 \> there exists $ a \in D_{i} $ such that \\
         \> $ \|av\,X \, Y  \|^{{H}^{M},g[s/S_\itype][a/Y_\itype]} = T
         $ and $ \| \lfloor \varphi  \rfloor Y
         \|^{{H}^{M},g[s/S_\itype][a/Y_\itype]} = F $  \\

	$\Leftrightarrow$ 
	\> $ V(\varphi) 
        \in  Iob_{\proptype \typearrow \proptype \typearrow o} (\|av\,
        S\|^{{H}^{M}, g [s/S_\itype] })$  \quad and \qquad \textbf{(similar to
        ***)}\\
	 \> there exists $ a \in D_{i} $ such that \\
         \> $ \|av\,X \, Y  \|^{{H}^{M},g[a/Y_\itype]} = T
         $ and $ \| \lfloor \varphi  \rfloor Y
         \|^{{H}^{M},g[a/Y_\itype]} = F $  \\

	$\Leftrightarrow$ 
	\> $ V(\varphi) 
        \in  Iob_{\proptype \typearrow \proptype \typearrow o} (av(s))$  \quad and \qquad \textbf{(similar to
        ***)}\\
	 \> there exists $ a \in D_{i} $ such that \\
         \> $ \|av\,X \, Y  \|^{{H}^{M},g[a/Y_\itype]} = T
         $ and $ \| \lfloor \varphi  \rfloor Y
         \|^{{H}^{M},g[a/Y_\itype]} = F $  \quad ($S\notin
        free(\lfloor\varphi \rfloor)$)\\

	$\Leftrightarrow$ 
	\> $ V(\varphi) 
        \in  ob(av(s))$  \quad and \\
	 \> there exists $ a \in S $ such that \\
        \> $ a \in av(s) $ and $M,a \not\models \varphi$ \quad (by
        induction hypothesis)\\

	$\Leftrightarrow$ 
	\> $ V(\varphi) 
        \in  ob(av(s))$  \quad and \\
	 \> there exists $ a \in S $ such that $ a \in av(s) $ and $ a\notin V(\varphi) $ \\

	$\Leftrightarrow$ 
	\> $ V(\varphi) \in ob(av(s)) $  \quad and \\
	\> there exists $ a \in S $ such that $ a\in  av(s)\cap V(\neg \varphi) $\\
	
	$\Leftrightarrow$ 
	\>  $ V(\varphi) \in ob(av(s)) $ and $ av(s)\cap V(\neg \varphi) \neq \emptyset $\\
	
	$\Leftrightarrow$ 
	\> $M,s\models \bigcirc_{a} (\varphi)$
\end{tabbing}

\noindent $\delta = \bigcirc_{p} (\varphi)$: \\[-1.5em]
\begin{tabbing}
	\qquad \=  The argument is analogous to $\delta = \bigcirc_{a} (\varphi)$. \qed
\end{tabbing}	
\end{proof}

\begin{lemma} \label{lemma4}
	For every Henkin model
	${H} = \langle \{D_\alpha\}_{\alpha \in {T}}, I \rangle$ such that
	${H}\models^\text{HOL} \Sigma $ for all 
	$ \Sigma\in \{\text{AV, PV1, PV2, OB1,..., OB5}\}$, there exists a
	corresponding 
	\cjl\ model $M$. Corresponding means that for all \cjl\ formulas $\delta $ and for all assignment $g$ and worlds $s$, 
	$ \| \lfloor \delta \rfloor S \|^{{H}, g [s/S] } = T
	$ if and only if 
	$ M,s \vDash \delta $.
\end{lemma} 

\begin{proof} % Throughout the proof  whenever possible we omit types in order to avoid making the notation too cumbersome. 
	Suppose that
	${H} = \langle \{D_\alpha\}_{\alpha \in {T}}, I \rangle$ is a
	Henkin model such that
	${H} \models^\text{HOL} \Sigma$ for all $\Sigma\in\{\text{AV, PV1, PV2,
		OB1,..,OB5}\} $. Without loss of generality, we can
              assume that the domains of $H$ are denumerable \cite{Henkin50}.
	We construct the corresponding \cjl\ model $ M $ as follows:

	\begin{itemize}[topsep=1pt,itemsep=0ex,partopsep=1ex,parsep=1ex]
		\item
		$ S= D_{i} $.  
		\item
		$ s \in av(u) $ for $ s,u \in S $ iff $Iav_{i\typearrow\proptype}(s,u) = T$.
		\item
		$ s \in pv(u) $ for $ s,u \in S $ iff $Ipv_{i\typearrow\proptype}(s,u) = T$.
		\item $\bar{X}\in ob(\bar{Y}) $ for
		% $ \bar{X}, \bar{Y} \in D_{\proptype}$ iff
                $\bar{X}, \bar{Y} \in D_i \longrightarrow  D_o$ iff
		$Iob_{\proptype \typearrow \proptype \typearrow o}
		(\bar{X},\bar{Y}) = T $.  %\marginpar{Say more?}
		\item
		$ s \in V(p^{j}) $ iff $Ip^{j}_\proptype(s) = T $ for all $p^{j}$.
	\end{itemize}

	Since ${H} \models^\text{HOL} \Sigma $ for all
	$ \Sigma \in \{\text{AV, PV1, PV2, OB1, .., OB5}\} $, it is
        straightforward (but tedious) to verify that $av$, $pv$ and $ob$ satisfy the 
	conditions as required for a \cjl\ model.

        Moreover, the above construction ensures that $H$ is a Henkin
        model $H^M$ for \cjl\ model $M$. Hence, Lemma 3 applies. This
        ensures that for all \cjl\ formulas $\delta $, for all
        assignment $g$ and all worlds $s$ we have
	$ \| \lfloor \delta \rfloor S \|^{{H}, g [s/S] } = T
	$ if and only if 
	$ M,s \vDash \delta $. \qed
\end{proof}

\begin{theorem}[Soundness and Completeness of the Embedding]
	\[\models^\text{\cjl} \varphi \text{ if and only if }
        \{\text{AV, PV1, PV2, OB1,..,OB5}\} \models^\text{HOL}
        \text{vld}\, \lfloor \varphi \rfloor\] 
\end{theorem}
     \begin{proof} (Soundness, $\leftarrow$) \sloppy The proof is by
		contraposition. Assume $\not\models^{\cjl} \varphi$,
                that is, there is
		a \cjl\ model $M = \langle S,av ,pv,ob,V \rangle$, and
                world $s\in S$, such that $M,s\not\models
                \varphi$. Now let $H^M$ be a Henkin model for \cjl\
                model $M$.
By Lemma \ref{lemma:3}, for an arbitrary assignment $g$, it holds that
		$\|\lfloor \varphi \rfloor\,S_\itype\|^{{H}^{M}, g 
			[s/S_\itype]} = F$.
		% ${V}(\lfloor g \rfloor
		%     [s/S_\itype],\lfloor \varphi \rfloor\,S_\itype) = F$ 
		% in Henkin model
		% ${H}^M= \langle \{D_\alpha\}_{\alpha \in {T}}, I \rangle$ for
		% $M$.
                Thus, by
		definition of $\|.\|$, it holds that
		$\|\forall S_\itype (\lfloor
		\varphi\rfloor\,S_\itype)\|^{{H}^{M},g} =
		\|\text{vld}\,\lfloor \varphi \rfloor\|^{{H}^{M}, g
			} =F$.
		% ${V}(\lfloor g \rfloor,\forall S_\itype (\lfloor \varphi
		% \rfloor\,S_\itype)) = {V}(\lfloor g \rfloor,\text{vld}\,\lfloor
		% \varphi \rfloor) = F$.
		Hence,
		${H}^M\not\models^\text{HOL} \text{vld}\,\lfloor \varphi \rfloor$. Furthermore,
		${H}^{M} \models^\text{HOL} \Sigma $ for all $ \Sigma
                \in \{\text{AV, PV1, PV2, OB1,\ldots,OB5}\}$ by Lemma
                2. Thus,
		$\{\text{AV, PV1, PV2, OB1,..,OB5}\}\not\models^\text{HOL} \text{vld}\,\lfloor
		\varphi \rfloor$.
		
		(Completeness, $\rightarrow$) The proof is again by
		contraposition. Assume\\
		$\{\text{AV, PV1, PV2, OB1,..,OB5}\}\not\models^\text{HOL} \text{vld}\, \lfloor
		\varphi \rfloor$, that is, there is a Henkin model
		${H}=\langle \{D_\alpha\}_{\alpha \in {T}}, I \rangle$ such that
		$H\models^\text{HOL} \Sigma $ for all $\Sigma \in
                \{\text{AV, PV1, PV2, OB1,..,OB5}\}$, but
		$\|\text{vld}\, \lfloor \varphi \rfloor\|^{{H},g} = F$
                for some
		assignment $g$ .
		% ${V}(\psi,\text{vld}\, \lfloor
		%     \varphi \rfloor) = F$.
		By Lemma 4, there is a \cjl\ model $M$ such that $ M
                \nvDash \varphi $. Hence, $\not\models^{\cjl}
                \varphi$. \qed

		%\footnote{Chris @ Valerio: This is even more handwaving here than in the IJCAI paper, do you think this is still ok?}
		%By definition of $\|.\|$ and since
		%$\text{vld}\,\lfloor \varphi \rfloor =_{\beta\eta} \forall
		%S_\itype (\lfloor \varphi \rfloor\,S_\itype)$
		%it holds
		%$\|\forall S_\itype (\lfloor \varphi
		%\rfloor\,S_\itype)\|^{{H}^{M}, g} = F$,
		% ${V}(\lfloor g \rfloor,\forall S_\itype
		%   (\lfloor \varphi \rfloor\,S_\itype)) = F$
		%and hence, by definition of $\text{vld}$,
		%$\|\lfloor \varphi \rfloor\,S_\itype\|^{{H}^{M},\lfloor g \rfloor
		%	[s/S_\itype]} = F$
		% ${V}(\lfloor g \rfloor [s/S_\itype],\lfloor \varphi
		%  \rfloor\,S_\itype) = F$
		%for some $s\in D_{i}$. By Lemma \ref{lemma:2}
		%$M,s \not \models \varphi$, and hence $\not\models^{\cjl} \varphi$.
	\end{proof}

Theorem 1 characterises \cjl\ as a natural fragment of HOL.

\section{Conclusion} \label{sec:conclusion} % The embeddings approach
% bridges between the Tarski view of logics (for meta-logic HOL) and
% the Kripke view (for the embedded source logics) and exploits the fact
% that well-known translations of logics, respectively, their
% Kripke-style semantical characterizations, can often be elegantly and
% directly formalized in HOL.
A shallow semantical embedding of Carmo and Jones's logic of contrary-to-duty conditionals 
in classical higher-order logic has been
presented, and shown to be faithfull (sound an complete). In addition, it
has meanwhile been implemented in the proof assistant Isabelle/HOL
(see the appendix).  This
implementation constitutes the first theorem prover for the logic  by
Carmo and Jones that is available to date. The foundational theory for
this implementation has been laid in this article.

There is
much room for future work. First, experiments could investigate
whether the provided implementation already supports non-trivial
applications in practical normative reasoning, or whether further
emendations and improvements are required. Second, the introduced
framework could also be used to systematically analyse the properties of
 Carmo and Jones's dyadic deontic logic within
Isabelle/HOL. Third, analogous to previous work in modal logic \cite{J23},
the provided framework could be extended to study and support first-order and
higher-order variants of the framework. 

% \begin{acknowledgements} 
%  We want to thank Tomer Libal for proof reading this article.
% \end{acknowledgements}

%\bibliography{chris,Bibliography,Bibliography2,Bibliography3}   % name your BibTeX data base

%\bibliographystyle{plain}
%\bibliographystyle{alpha}
%\bibliographystyle{unsrt}
\bibliographystyle{plain}
%\bibliography{chris,Bibliography,Bibliography2,Bibliography3,CJL}   % name your BibTeX data base
\bibliography{CJL} 

\newpage
\begin{appendix}
\section{Implementation in Isabelle/HOL}
The semantical embedding as devised in this article has been
implemented in the higher-order proof assistant Isabelle/HOL \cite{Isabelle}. 
Figure~\ref{fig1} displays the respective encoding. Figures  ~\ref{fig2} and
~\ref{fig3} report on some experiments.  

\begin{figure}[ht]
\includegraphics[width=\textwidth]{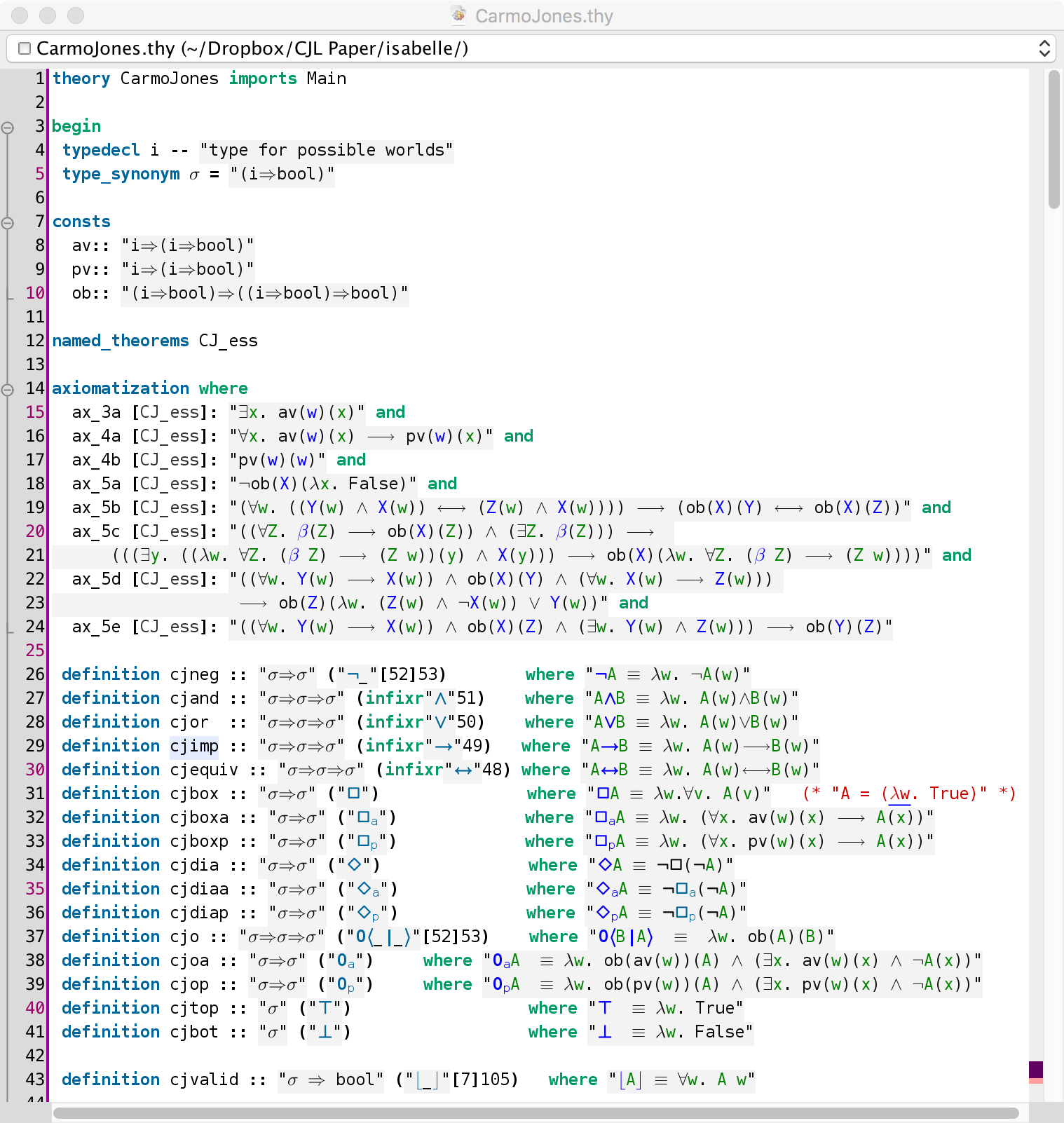}
\caption{Shallow semantical embedding of \cjl\ in Isabelle/HOL. \label{fig1}}
\end{figure}

\begin{figure}[ht]
\includegraphics[width=\textwidth]{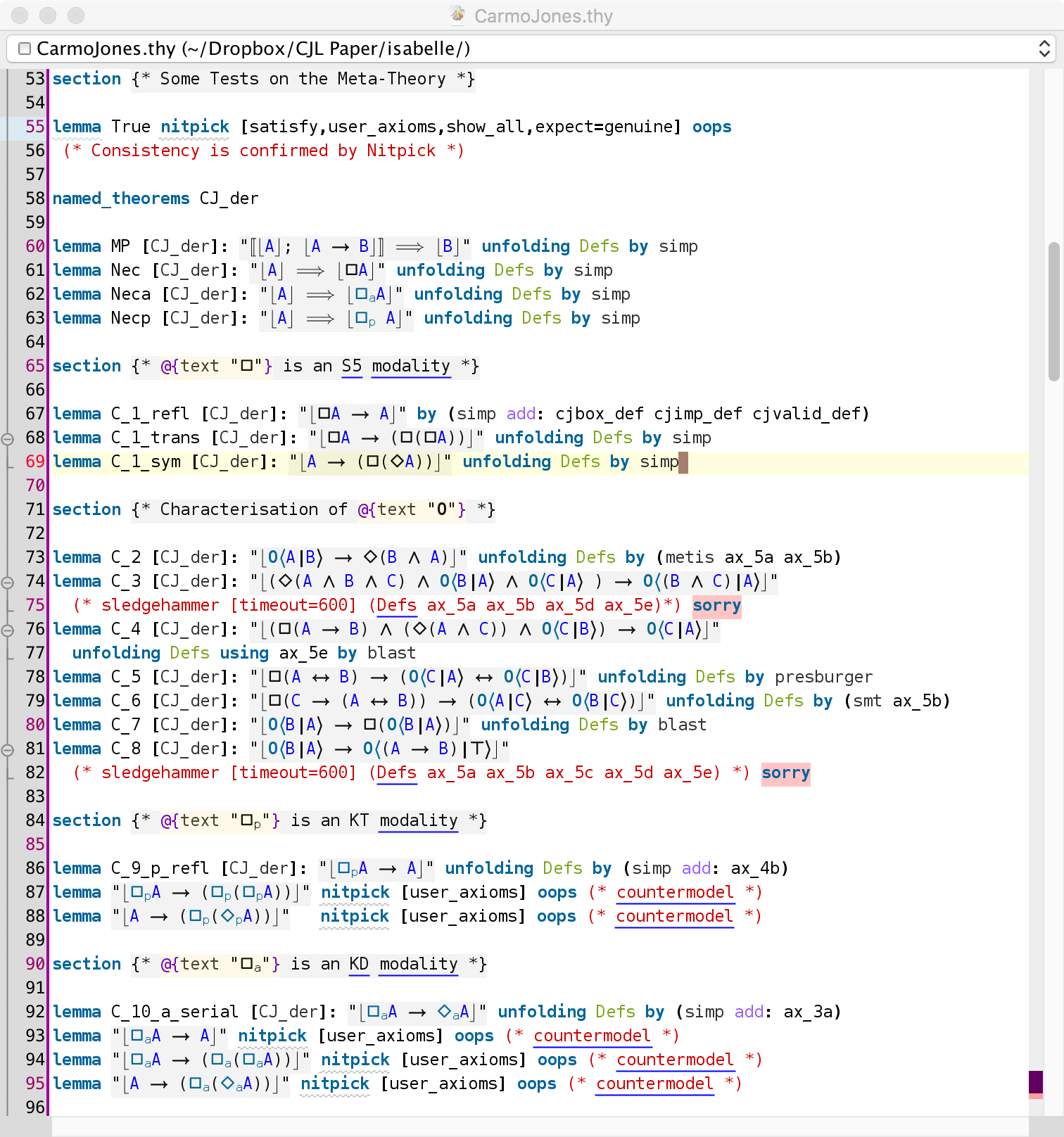}
\caption{Experiments (meta-theory) with the embedding of \cjl\ in
  Isabelle/HOL. In the ``sorry'' cases proofs can be automatically
  found by theorem provers integrated via Sledgehammer, but a
  reconstruction of these proofs in Isabelle/HOL still fails, since
  the internal provers are too weak.  \label{fig2}}
\end{figure}

\begin{figure}[ht]
\includegraphics[width=\textwidth]{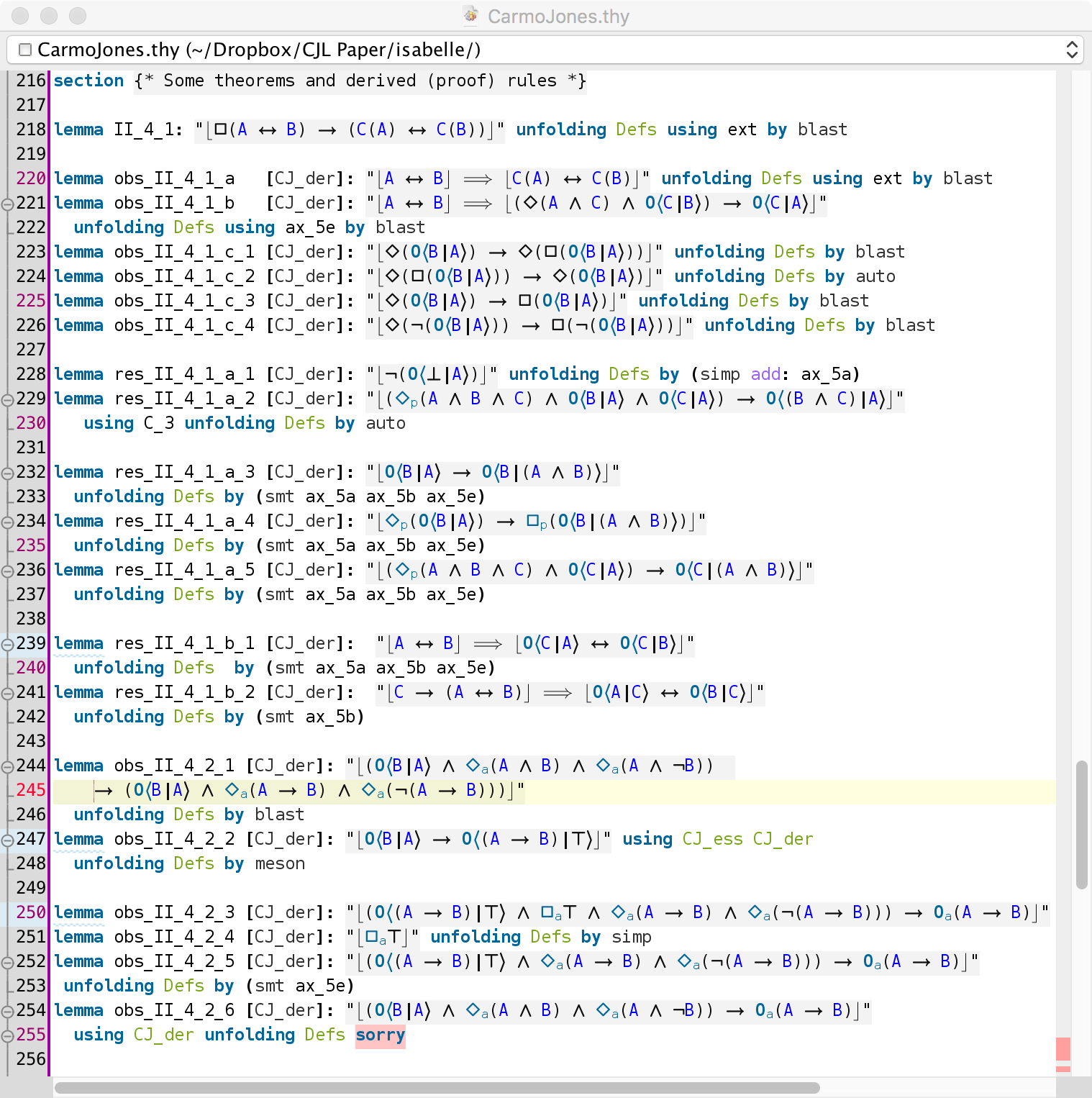}
\caption{Further experiments (lemmata and derived rules) with the embedding of \cjl\ in
  Isabelle/HOL. \label{fig3}}
\end{figure}

\end{appendix}

\end{document}